\documentclass[runningheads]{llncs}
\usepackage[T1]{fontenc}
\usepackage{graphicx}
\usepackage[numbers]{natbib}
\usepackage{multicol}
\usepackage[bookmarks=true]{hyperref}
\usepackage{xspace}
\usepackage{tikz}
\usetikzlibrary{shapes, arrows.meta, positioning, calc, fit, backgrounds}

\usepackage[dvipsnames]{xcolor}
\usepackage{amssymb}
\usepackage{nicefrac}
\usepackage{amsmath}
\usepackage{cuted}   
\usepackage{caption} 
\usepackage{float}   
\usepackage{graphicx}

\usepackage{algorithm}      
\usepackage[noend]{algpseudocode}  

\algrenewcommand\algorithmicensure{\textbf{Input:}}

\usepackage[bookmarks=true]{hyperref}
\usepackage{cleveref}

\usepackage[numbers]{natbib}
\usepackage{multicol}

\usepackage{tabularx}
\usepackage{booktabs}
\usepackage{multirow}
\usepackage{subcaption}

\usepackage{standalone} 
\usepackage{tikz}

\newcommand{\adir}[1]{\textcolor{blue}{(\textbf{Adir:} #1)}}
\newcommand{\kiril}[1]{\textcolor{red}{(\textbf{Kiril:} #1)}}
\newcommand{\oren}[1]{\textcolor{teal}{(\textbf{OS:} #1)}}

\newcommand{\rewrite}[1]{\textcolor{darkgreen}{(\textbf{Rewrite:} #1)}}
\newcommand{\former}[1]{\textcolor{orange}{(\textbf{Former:} #1)}}

\newcommand{\ip}{\textsf{IP}\xspace}

\newcommand{\gip}{\textsf{GIP}\xspace}

\newcommand{\gtsp}{\textsf{Group-TSP}\xspace}
\newcommand{\gst}{\textsf{Group-ST}\xspace}

\def\niceparagraph#1{\vspace{3pt} \noindent \textbf{#1}}

\spnewtheorem{problem}{Problem}{\bfseries}{\itshape}

\newcommand{\classP}{\ensuremath{\mathsf{P}}}
\newcommand{\classNP}{\ensuremath{\mathsf{NP}}}

\newcommand{\ignore}[1]{}

\newcommand{\Cpp}{C\raise.08ex\hbox{\tt ++}\xspace}

\newcommand\algname[1]{\textsf{#1}\xspace}
\definecolor{darkgreen}{RGB}{0,128,0}

\begin{document}

\title{Scalable Inspection Planning via Flow-based Mixed Integer Linear Programming}

\author{Adir Morgan\thanks{Corresponding author.} \and Kiril Solovey \and Oren Salzman}
\institute{Technion--Israel Institute of Technology, Haifa, Israel\\
\email{samorgan@campus.technion.ac.il, kirilsol@technion.ac.il, osalzman@cs.technion.ac.il}}
\authorrunning{Morgan, Solovey, and Salzman}

\titlerunning{Scalable Inspection Planning via Flow-based MILP}

\maketitle

\begin{abstract}
\emph{Inspection planning} is concerned with computing the shortest robot path to inspect a given set of points of interest (POIs) using the robot's sensors. This problem arises in a wide range of applications from manufacturing to medical robotics. 
To alleviate the problem's complexity, recent methods rely on sampling-based methods to obtain a more manageable (discrete) \emph{graph inspection planning} (\gip) problem. Unfortunately, \gip still remains highly difficult to solve at scale as it requires simultaneously satisfying POI-coverage and path-connectivity constraints, giving rise to a challenging optimization problem, particularly at scales encountered in real-world scenarios.  
In this work, we present highly scalable Mixed Integer Linear Programming (MILP) solutions for \gip that significantly advance the state-of-the-art in both runtime and solution quality. Our key insight is a reformulation of the problem’s core constraints as a network flow, which enables effective MILP models and a specialized Branch-and-Cut solver that exploits the combinatorial structure of flows.
We evaluate our approach on medical and infrastructure benchmarks alongside large-scale synthetic instances. Across all scenarios, our method produces substantially tighter lower bounds than existing formulations, reducing optimality gaps by 30–50\% on large instances. Furthermore, our solver demonstrates unprecedented scalability: it provides non-trivial solutions for problems with up to \emph{15,000 vertices} and thousands of POIs, where prior state-of-the-art methods typically exhaust memory or fail to provide any meaningful optimality guarantees.

\end{abstract}


\section{Introduction \& Related Work}

In inspection planning (\ip), a robot equipped with an onboard sensor is tasked with computing a path in a known environment to inspect a set of points of interest (POIs)
while avoiding obstacles and minimizing path cost.
Applications of \ip arise in a wide range of domains, including construction~\cite{cheng2008time}, manufacturing~\cite{atkar2005uniform}, and medical robotics~\cite{cho2024efficient,cho2021planning}.

State-of-the-art approaches to the \ip problem~\cite{fu2019toward,fu2021computationally} use sampling-based motion planning to discretize the robot’s continuous configuration space into a roadmap represented as a graph. The vertices and edges of this graph correspond to valid robot configurations and feasible local motions, respectively. Edge costs encode the cost of local motion, and directed edges may be used to capture asymmetric motions arising from kinodynamic constraints. Each vertex is additionally associated with a \emph{visibility set} that specifies which POIs can be inspected from that configuration. The \ip task then reduces to finding a minimum-cost tour on this graph that inspects all POIs, a problem known as \emph{graph inspection planning} (\gip).
\gip presents a significant computational challenge, as it entails the simultaneous selection of a subset of vertices that jointly cover all POIs, together with the computation of a minimum-length tour connecting them. Thus, \gip generalizes both the \emph{set cover} problem and the \emph{traveling salesman problem} (TSP)~\cite{karp1972reducibility, lawler1985tsp},  both of which are NP-hard, and even hard to approximate within constant factors~\cite{feige1998threshold, sahni1976p}.

\citet{fu2019toward} proposed a search-based approach that reformulates \gip as a shortest-path problem on a new graph $G_{\ip}$, whose vertex set captures all original roadmap vertices and all possible POI subsets. Consequently, the size of~$G_{\ip}$ is exponential in the number of POIs. While this approach provides theoretical guarantees on solution quality, 
it does not manage to handle  real-world instances involving thousands of POIs within a reasonable time frame, thus limiting its practical applicability.
More recently, \citet{mizutani2024leveraging} tackled \gip using dynamic programming and mixed-integer linear programming (MILP) approaches. 
Among the two, the most scalable is their MILP formulation, solved using the Branch-and-Bound (BnB) paradigm (see Sec.~\ref{sec:alg}).
While this formulation can handle substantially larger \gip instances than
earlier methods, its performance degrades as the graph size and the number of POIs
grow, leading to loose lower bounds and limited near-optimality certification,
even for modest graph sizes (relative to those commonly considered
in motion planning~\cite{Panasoff.Solovey.25}).
This motivates 
alternative formulations that retain strong solution-quality guarantees while enabling scalability.

\begin{figure*}[t]
    \hspace*{-0.8cm}
    \centering
    \subfloat[]{\includegraphics[height=1.9cm]{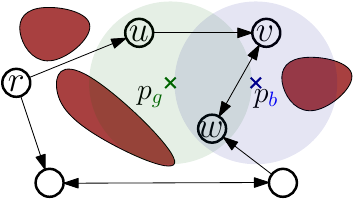}}
    \hspace{4mm}
    \subfloat[]{\includegraphics[height=1.9cm]{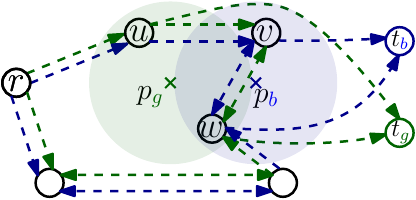}}
    \hspace{4mm}
    \subfloat[]{\includegraphics[height=1.9cm]{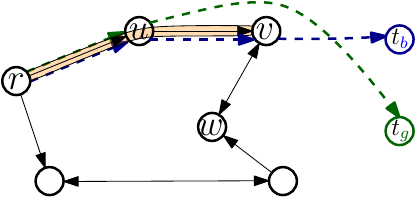}}
  \caption{
  Graph inspection planning, with network flow reduction for inspecting tour generation: (a)~A given inspection planning graph, where The POI
  $p_{\color{blue} b}$ can be inspected from  vertices~$v$ and~$w$,
  and the POI
  $p_{\color{darkgreen} g}$ can be inspected from~$u$ and~$w$ (depicted by colored circles surrounding each POI).
  (b)~Pseudo-terminals $t_{\color{blue} b}$ and $p_{\color{green} g}$ are introduced, and a multi-commodity network flow problem is formulated (dashed colored edges), where providing positive flow from the root $r$ to each pseudo-terminal will correspond to a path from the root to visit an inspecting configuration for the associated POI.
  (c)~A specific flow solution (orange edges), minimizing the cumulative path cost to inspect all POIs. Additional constraints will handle path returning to root.
  }
 \label{fig:flow-interp}
\end{figure*}

\niceparagraph{Contribution.}
We develop MILP-based \gip solvers that scale to large instances while providing tight bounds on solution quality.
Our approach is driven by a flow-based interpretation of \gip, illustrated in Fig.~\ref{fig:flow-interp}, which demonstrates how coverage and connectivity constraints can be expressed through network-flow structures within MILP formulations.
Toward this goal, we start by formally defining the \gip problem (Sec.~\ref{sec:problem-def}) 
and then reformulate it using a group-covering perspective (Sec.~\ref{sec:group-covering}). 
This perspective situates \gip within a well-studied class of combinatorial optimization problems and yields structural insights that guide the design of effective MILP solvers.
Using this perspective, we present a baseline MILP formulation that enforces coverage and local structure (Sec.~\ref{gip-form}), yet lacks global connectivity enforcement. We next propose three different approaches to ensure global connectivity (Sec.~\ref{gip-sec-poly}), each offering a trade-off between formulation compactness and solver strength.

Our most scalable approach out of the three relies on network-flow cutset constraints, that are too numerous to be explicitly evaluated simultaneously. Instead, we consider a lazy constraint-generation approach with the Branch-and-Cut framework~\cite{mitchell2002branch,padberg1991branch}.
%
To construct this solver, we introduce in Sec.~\ref{sec:alg} tailored algorithmic building blocks, that enable on-demand constraint generation and the construction of progressively-improving feasible solutions to \gip.
We conclude (Sec.~\ref{sec:evaluation}) with an extensive experimental evaluation on real-world and simulated instances spanning a wide range of problem scales, demonstrating the scalability of the proposed approach and the trade-offs between different formulations. 

\vspace{-5pt}
\section{Problem Definition}
\label{sec:problem-def}
Let $G=(V,E,c)$ be a weighted directed graph that abstracts the robot's motion, where vertices correspond to robot configurations and edges correspond to dynamically feasible, collision-free motions with cost $c(e)$.
Let $\mathcal{P}$ denote a set of points of interest (POIs) in the workspace.
The robot is equipped with an exteroceptive sensor (e.g. camera, lidar), inducing a \emph{coverage function} $\chi:V \rightarrow 2^{\mathcal{P}}$ that specifies the POIs inspected when the robot is positioned at a given vertex.

Given a designated root vertex $r \in V$, a \emph{tour} $\tau$ is a closed path starting and ending at $r$, traversing edges in $E$.
The cost of a tour is the sum of its edge costs, denoted $c(\tau)$, and its coverage is the union of POIs observed along the tour, denoted $\chi(\tau)$.
Let $\mathcal{T}(G,r)$ denote the set of all such tours.

\begin{definition}[Graph inspection planning]
\label{def:gipp}
Given $\langle G=(V,E,c), r, \mathcal{P}, \chi \rangle$, find a minimum-cost tour from $r$ that inspects all POIs, i.e.,
\[
\arg \min_{\tau \in \mathcal{T}(G,r)} c(\tau)
\quad \text{subject to} \quad
\chi(\tau) = \mathcal{P}.
\]
\end{definition}

\vspace{-15pt}
\section{Graph Inspection Planning  via Group Covering}
\label{sec:group-covering}
To set the foundations for the design of scalable MILP formulations for \gip, we reinterpret the problem through a group-covering (GC) lens, revealing similarities to group variants of 
\algname{TSP}~\cite{applegate2011traveling} and steiner tree (\algname{ST})~\cite{hwang1992steiner}.
In GC  problems~\cite{garg2000polylogarithmic, laporte1983generalized}, alongside a weighted graph $G=(V,E,c)$ we are also given a set~$\mathcal{S} = \{S_1, S_2, \dots, S_k\}$ of $k\geq 1$ subsets (or ``groups'') of vertices~$S_i \subseteq V$. 
The goal is to visit the vertices of~$V$, according to some path constraints, such that at least one vertex in each set~$S_i$ is visited. 
To express \gip as a GC problem, recall that in \gip each vertex~$v \in V$ covers a set of POIs~$\chi(v) \subseteq \mathcal{P}$.
We define the associated family of groups by \emph{inverting}~$\chi$: for each~$p \in \mathcal{P}$, define a set~$
    S_p := \{ v \in V \mid p \in \chi(v) \},
$
and let $\mathcal{S}:=\{S_p\}_{p\in \mathcal{P}}$.
Here, $S_p$ consists of all vertices from which POI $p$ can be inspected,
and inspecting $p$ in a tour corresponds to visiting at least one vertex in $S_p$. Thus, \gip can be expressed as follows.

\begin{definition}[\textsf{Group-covering form of \gip}]
\label{def:group-gipp}
Given $\langle G=(V,E,c), r, \mathcal{S} \rangle$,
find a minimum-cost tour from $r$ that covers all groups in $\mathcal{S}$.
\[
    \operatorname*{arg\,min}_{\tau \in \mathcal{T}(G,r)} \; c(\tau)
    \quad \text{subject to} \quad
    \forall\, S_i \in \mathcal{S}: \; V(\tau) \cap S_i \neq \emptyset .
\]
\end{definition}

This representation reveals a connection to two well-studied GC
problems: 

\begin{definition}[\gtsp~\cite{laporte1983generalized}]
\label{def:gtsp}
Given  $\langle G=(V,E,c), r, \mathcal{S} \rangle$,
let $\mathcal{T}_1(G,r)\subseteq \mathcal{T}(G,r)$ denote tours from $r$ that visit any vertex in $V\setminus \{r\}$ at most once, 
find a minimum-cost tour in $\mathcal{T}_1(G,r)$
that covers $\mathcal{S}$. 
\[
    \operatorname*{arg\,min}_{\tau \in \mathcal{T}_1(G,r)} \; c(\tau)
    \quad \text{subject to} \quad
    \forall\, S_i \in \mathcal{S}: \; V(\tau) \cap S_i \neq \emptyset .
\]
\end{definition}
%
\begin{definition}[\gst~~\cite{garg2000polylogarithmic}]
\label{def:gst}
Given $\langle G=(V,E,c), r, \mathcal{S} \rangle$,
find a minimum-cost tree $T\in \textup{Trees}(G,r)$, where $\textup{Trees}(G,r)$ is the set of all subtrees of $G$ rooted in $r$,  that covers all groups in $\mathcal{S}$.
\[
    \operatorname*{arg\,min}_{T \in \textup{Trees}(G,r)} \; c(T)
    \quad \text{subject to} \quad
    \forall\, S_i \in \mathcal{S}: \; V(T) \cap S_i \neq \emptyset .
\]
\end{definition}

Despite their apparent similarity, the path-simplicity constraint in \gtsp introduces additional combinatorial complexity, distinguishing it from \gip, where vertex revisits are allowed.
In contrast, \gst is structurally closer to \gip, as it emphasizes connectivity between the root and representative vertices of each group rather than tour simplicity.
However, existing \gst approaches typically rely on reductions to classical Steiner Tree~\cite{gamrath2017scip}, which enlarge the problem and limit scalability. Additionally, their conversion to tours can yield poor inspection paths, especially in directed graphs.
These observations motivate MILP formulations for \gip that, inspired by \gst, enforce root-to-group connectivity rather than explicit tour ordering.
They further imply that approximation bounds for \gst translate to \gip up to constant factors, yielding state-of-the-art \gip approximation bounds, we derive in App.~\ref{sec:Approx-bounds}.

\section{MILP Formulations for Graph Inspection Planning}
\label{sec:milp-for-gip}
Following the group-covering perspective (Sec.~\ref{sec:group-covering}), we propose MILP formulations for \gip that enforce both group coverage and tour structure.
We start with a baseline formulation that enforces coverage and local connectivity, 
and then introduce additional constraints to ensure global connectivity.

We denote the number of POIs by $|\mathcal{P}| = |\mathcal{S}| = k$, and set $K = \{1,\dots,k\}$. 
Furthermore, 
$N^+(v) := \{u \in V : (v,u)\in E\}$
and
$N^-(v) := \{u \in V : (u,v)\in E\}$
denote the out-neighborhood 
and in-neighborhood
of vertex  $v$, respectively.

\subsection{Baseline MILP Formulation} \label{gip-form}
The following is our baseline MILP formulation of \gip{\footnote{Although some instances involve only integer variables, we consistently use the term MILP rather than ILP, as our solution methodology does not distinguish between the two.}. 
The inspection tour~$\tau$ is specified through binary decision variables $x_{uv}$ associated with each edge $(u,v)\in E$, i.e.,~$x_{uv}=1$ indicates that $\tau$ traverses edge $(u,v)$.
\begin{subequations}\label{eq:gip-baseline}
{\small 
\begin{alignat}{2}
\min_{x_{uv},\forall {(u,v)\in E}} \quad &
\sum_{(u,v)\in E} c(u,v)\cdot x_{uv}
\label{eq:gtsp-milp-obj}
\\[0.1em]
\text{s.t.}\quad &
\sum_{v\in N^+(r)} x_{rv} \ge 1,
\quad \label{eq:gtsp-root}
\\[0.1em]
&
\sum_{v\in S_i} \sum_{u\in N^-(v)} x_{uv} \ge 1,
\quad  \forall S_i\in\mathcal{S}, \label{eq:gtsp-groups}
\\[0.1em]
&
\sum_{u\in N^-(v)} x_{uv} = \sum_{u\in N^+(v)} x_{vu},
\quad  \forall v\in V, \label{eq:gtsp-balance}
\\[0.1em]
&
x_{uv} \in \{0,1\},
\quad  \forall (u,v)\in E.\label{eq:gtsp-binary}
\end{alignat}}
\end{subequations}
Objective~\eqref{eq:gtsp-milp-obj} minimizes the total cost of the selected edges,
and
constraints~\eqref{eq:gtsp-root} and~\eqref{eq:gtsp-groups} 
ensure that the corresponding tour will start from~$r$ and cover every group $S_i\in\mathcal{S}$, respectively.\footnote{This formulation permits revisiting vertices but explicitly forbids revisiting edges. However, the latter is not limiting: \citet{mizutani2024leveraging} (Lemma~1) prove that any optimal \gip tour traverses each directed edge at most once.
Additionally, while outside the scope of this work, we note that this formulation can be extended to partial-coverage \gip~\cite{fu2019toward, mizutani2024leveraging}. See App.~\ref{sec:partial-cover}.}
Finally, constraints~\eqref{eq:gtsp-balance} enforce that every vertex in the subgraph induced by the selected edges has equal in-degree and out-degree.
As a result, the selected edges decompose into one or more edge-disjoint
directed cycles~\cite{bondy1979graph}.
Additional subtour-elimination constraints (SEC) are therefore required to
ensure that all selected edges belong to a single connected tour containing
the root.
Before introducing our SECs, we briefly review the role of LP-relaxations in Branch-and-Bound (BnB)~\cite{lawler1966branch}, and their interpretation for~\gip.

\niceparagraph{Linear-programming (LP) relaxation.}
In BnB, the MILP's feasible region   is recursively partitioned into a sequence of increasingly restricted subproblems. For each subproblem, an \emph{LP relaxation} of the MILP, obtained by omitting integrality constraints (i.e., turning~$x_{uv}\in \{0,1\}$ into $x_{uv}\in[0,1]$), is solved by an LP-solver. The objective value of this relaxed problem provides a lower bound for any feasible integer solution within the corresponding subproblem. 

\niceparagraph{Flow-based interpretation.}
In \gip, the LP relaxation admits a natural flow interpretation, where each variable $x_{uv}\in[0,1]$ is viewed as assigning a fractional amount of flow to edge $(u,v)$. 
Feasible LP solutions therefore correspond to routing fractional flow that satisfies group-coverage, motivating our flow-based \gip formulations.

\niceparagraph{Integrality gap.}
In addition to lower bounds obtained from LP relaxations, a BnB solver maintains feasible \emph{integer} solutions that provide upper bounds.
Convergence is achieved as these bounds progressively tighten.
The effectiveness of a MILP formulation is therefore governed by its
\emph{integrality gap}, defined as the difference between the LP optimum and the best integer-feasible solution.
Large integrality gaps lead to weak bounds and slow convergence in BnB,
highlighting the importance of strong subtour-elimination constraints~\cite{land2009automatic}.

\subsection{Subtour Elimination Constraints (SEC) for \gip} \label{gip-sec-poly}
Motivated by the flow-based interpretation of the LP relaxation, we introduce three families of subtour-elimination constraints (SEC), where the binary edge-selection variables $x_{uv}$ serve as capacities for a directed network, while additional continuous flow-variables based mechanisms are used to enforce global connectivity constraints.



%

\niceparagraph{Single-commodity flow (SCF).}
We consider a single flow originating from the root which is propagated along and consumed by any selected edges, thereby enforcing connectivity. This is inspired by similar approaches for TSP solvers~\cite{gavish1978travelling}.
Specifically, we introduce a nonnegative continuous flow variable $f_{uv} \ge 0$ for each directed edge $(u,v) \in E$, together with the following constraints:
\begin{subequations}\label{eq:SCF}
{\small
\begin{alignat}{1}
f_{uv} &\le M\cdot x_{uv},
\quad \forall (u,v)\in E,
\label{eq:SCF-a}
\\[0.1em]
\sum_{u\in N^-(v)} f_{uv}
- \sum_{u\in N^+(v)} f_{vu}
&= \sum_{u\in N^+(v)} x_{vu},
\quad \forall v\in V\setminus\{r\}.
\label{eq:SCF-b}
\end{alignat}}
\end{subequations}

Constraint~\eqref{eq:SCF-a} ensures flow is permitted only on selected edges,
where $M$ is chosen sufficiently large to accommodate any required flow. Following Lemma~2 of~\citet{mizutani2024leveraging}, which (tightly) bounds the length of an optimal tour by $2\cdot (|V|-1)$, we set $M = 2\cdot (|V|-1)$. 
Constraint~\eqref{eq:SCF-b} forms a flow-consumption rule, which eliminate subtours by requiring each traversal of a selected edge to consume one unit of flow. This is expressed from a vertex-centric perspective: A visit to a vertex $v$ is induced by selecting an incoming edge $x_{uv}$ for some $u \in N^-(v)$. Each such visit consumes one unit of flow, implying that the total incoming flow at $v$ exceeds the total outgoing flow by exactly one unit per visit.

\begin{lemma}[SCF constraints eliminate subtours]
    An optimal solution for the MILP formulation defined by constraints \eqref{eq:gip-baseline} and \eqref{eq:SCF} yields a single tour containing $r$.
\end{lemma}
\begin{proof}\emph{(sketch)}
    By contradiction, assume that the optimal solution contains a directed subtour that does not visit the root $r$. Denote by $C\subseteq V$ the vertices visited by this subtour. 
    Summing constraint \eqref{eq:SCF-b} over all the vertices in $C$, the left-hand side telescopes to zero, since every unit of flow leaves a cycle-vertex enters another cycle-vertex, and the net-flow out of $C$ is zero. In contrast, the right-hand side sums to the number of edges in the cycle, which is strictly positive. This yields a contradiction,
showing that such a cycle cannot satisfy the flow constraints. 
The only vertex excluded from constraint~\eqref{eq:SCF-b} is the root, which is therefore allowed to exhibit a net flow imbalance, effectively acting as a flow source. Consequently, any feasible cycle must include the root.
\qed
\end{proof}

The SCF formulation augments the baseline with constraints~\eqref{eq:SCF}, introducing~$2|E|$ additional continuous variables and $O(|E|+|V|)$ constraints, and is the most compact MILP formulation we consider.
However, for the LP-relaxed problem, the Big-$M$ constraints in~\eqref{eq:SCF-a} allow substantial flow on weakly selected edges, enabling the LP relaxation to satisfy connectivity constraints without committing to specific edges.
This results in a large integrality gap, motivating tighter flow formulations that avoid Big-$M$ constants, which we introduce next.

\citet{mizutani2024leveraging} proposed an alternative subtour-elimination scheme based on continuous \emph{charge} variables, which enforces connectivity by maintaining a global charge balance with the root acting as a sink.
Although asymptotically as compact as SCF, this formulation induces a different LP relaxation.
We empirically compare their behavior within BnB in Sec.~\ref{sec:evaluation}.

\niceparagraph{Multi-commodity flow (MCF).}
To  overcome the large integrality gap that may be induced by constraint~\eqref{eq:SCF}, we consider the following SEC mechanism, inspired by related routing and network-design problems~\cite{claus1984new, wong1980integer}.
The key idea, depicted in Fig.~\ref{fig:flow-interp}, is to associate a distinct flow commodity with each group $i\in K$, and to require that one unit of flow be delivered from the root to a vertex covering that group. 
Formally, we introduce continuous flow variables
$f^i_{uv} \in [0,1]$ for each group $i \in K$ and edge $(u,v) \in E$.
Each group $i$ requires one unit of flow to originate at the root $r$.
\begin{subequations}\label{eq:MCF}
{\small
\begin{alignat}{2}
&
\sum_{v \in N^+(r)} f^i_{rv} \ge 1,
\quad && \forall i \in K,
\label{eq:MCF-a}
\\[0.1em]
&
\sum_{v \in S_i} \sum_{u \in N^-(v)} f^i_{uv}
- \sum_{v \in S_i} \sum_{u \in N^+(v)} f^i_{vu} \ge 1,
\quad && \forall i \in K,
\label{eq:MCF-b}
\\[0.1em]
&
\sum_{u \in N^-(v)} f^i_{uv}
= \sum_{u \in N^+(v)} f^i_{vu},
\quad && \forall i \in K,\; v \in V \setminus (S_i \cup \{r\}),
\label{eq:MCF-c}
\\[0.1em]
&
f^i_{uv} \le x_{uv},
\quad && \forall i \in K,\; (u,v) \in E.
\label{eq:MCF-d}
\end{alignat}}
\end{subequations}


In the MCF formulation, each group requires one unit of flow from the root to a covering vertex, ensuring that any cycle carrying flow is connected to the root by flow conservation~\eqref{eq:MCF-c}. Cycles without commodity flow are irrelevant and can be removed without increasing cost, yielding a single structure rooted at the start vertex.
By bounding flow directly with edge-selection variables~\eqref{eq:MCF-d}, MCF avoids Big-$M$ constraints and produces a much tighter LP relaxation than SCF.
However, this strength comes at a high cost: the formulation adds $2k|E|$ continuous variables, leading to tens of millions of variables in realistic instances~\cite{fu2019toward,fu2021computationally,mizutani2024leveraging}, which makes it impractical beyond small scales (App.~\ref{appendix:small-exp}).

%
%

Consequently, we next move on to suggest a formulation that retains this connectivity strength without introducing an explicit flow commodity for each group. 
While reducing the number of variables, it will introduce an exponential number of constraints. 
To avoid generating them explicitly, which will render the approach infeasible, we will take a lazy approach and generate them on demand.

\niceparagraph{Group-cutset formulation.} \label{gip-gcutset}
We present a subtour-elimination mechanism based on explicit connectivity constraints, inspired by cutset-based approaches for \algname{TSP}~\cite{applegate2001tsp}, which explicitly enforces root-to-group connectivity.
This formulation aims to achieve the structural strength of root-to-group connectivity enforcement found in multi-commodity flow SEC, while avoiding its prohibitive memory requirements, thereby enabling application to large-scale \gip instances. We achieve this by introducing an exponential number of vertex-cut constraints, that are \emph{lazily evaluated} within a Branch-and-Cut framework. 
The effectiveness of this approach relies on the fact that these \emph{group-cutset} constraints can be efficiently validated using network-flow based algorithms, as detailed in Sec.~\ref{sec:alg}.


Formally, for any group $S_i \in \mathcal{S}$, we examine all partitions of the vertex set $V$ into two disjoint subsets $R \subset V$ and $V \setminus R$ such that the root $r \in R$ and $R \cap S_i = \emptyset$.
Such a partition defines a vertex cut  between $R$ and its complement.
If a tour were to remain entirely within $R$, it would be impossible to visit any vertex in~$S_i$, and thus the corresponding POI $p_i$ could not be inspected. Therefore, any feasible \gip tour must include at least one edge crossing this vertex cut. Thus, we add constraints requiring the selected edges to cross every such cut~$R$, enforcing connectivity between the root and each group $S_i$.
Formally, let~$\mathcal{R}$ denote the family of vertex subsets that contain $r$ but exclude at least one group, i.e.,~$
\mathcal{R} := \left\{ R \subseteq V \;\middle|\; r \in R \;\text{and}\; \exists S_i \in \mathcal{S}
\text{ such that } R \cap S_i = \emptyset \right\}$.
For any~$R \subseteq V$, we define its outgoing edge set as
$\delta^{+}(R) := \{(u,v) \in E \mid u \in R,\; v \notin R\}$.
The baseline formulation~\eqref{eq:gip-baseline} is 
augmented with the following family of \emph{group-cutset constraints}, requiring that for any $R\in \mathcal{R}$, at least one selected edge  leaves
$R$:
\begin{equation}
\label{eq:group-cutset}
\textstyle \sum_{(u,v)\in \delta^{+}(R)} x_{uv} \ge 1, \qquad \forall R \in \mathcal{R}.
\end{equation}

The following lemma states that this constraint indeed eliminates subtours.

\begin{lemma}[Group-cutset constraints eliminate subtours]
\label{lem:group_cutset_sec}
Assume $c(e) > 0$, for all $e \in E$.\footnote{
If zero-cost edges are allowed, i.e., $c(e)\ge 0$, an optimal solution may
contain additional disconnected components of zero total cost.
Such components can be removed without affecting feasibility or optimality. 
} 
The MILP formulation defined by constraints \eqref{eq:gip-baseline} and \eqref{eq:group-cutset} yields a single tour containing $r$.  
\end{lemma}

\begin{proof}
Let $\{x^*\}$ be an optimal integral solution and set
$F := \{e \in E \mid x^*_e = 1\}$. We show that $F$ forms a single tour containing the root. 
Let $C_r \subseteq V$ denote the strongly-connected component of $r$ in the
subgraph induced by $F$, and let $F_r \subseteq F$ be the edges 
whose both ends are in~$C_r$. Note that $|C_r|\ge 2$ due to constraint~\eqref{eq:gtsp-root}. 
Define a new solution $\{\hat{x}\}$ by setting $\hat{x}_e = 1$ if $e \in F_r$ and
$\hat{x}_e = 0$ otherwise and note  that it can be easily shown that $\{\hat{x}\}$ is feasible.
\ignore{Constraint~\eqref{eq:gtsp-root} holds since $r \in C_r$ and therefore retains at
least one outgoing selected edge.
For each group $S_i \in \mathcal{S}$, the group-cutset constraints~\eqref{eq:group-cutset} imply that
$S_i$ is reachable from $r$ in the subgraph induced by $F$, and hence
$S_i \cap C_r \neq \emptyset$.
Thus, $\hat{x}$ satisfies the group-coverage constraints
\eqref{eq:gtsp-groups}.
Flow-balance constraints~\eqref{eq:gtsp-balance} remain satisfied, as removing
edges outside $C_r$ does not affect balance at vertices within $C_r$.
Finally, $\hat{x}$ is integral by construction.}

If $F \neq F_r$, then $\{\hat{x}\}$ removes at least one selected edge.
Since all edge costs are strictly positive, this strictly decreases the objective
value, contradicting the optimality of $\{x^*\}$.
Therefore, $F = F_r$, and all selected edges lie in the connected component of $r$. \qed 
\end{proof}

As $\mathcal{R}$ grows exponentially with the graph size, explicitly enumerating the group-cutset formulation would render the solver intractable. However, given a candidate solution, constraints~\eqref{eq:group-cutset} can be efficiently \emph{verified}, and for infeasible candidates, violated constraints can be identified using efficient flow-based algorithms. 
This property allows us to apply the \emph{Branch-and-Cut} (BnC) approach enabling a \emph{lazy} evaluation of the constraints \eqref{eq:group-cutset}, hence avoiding explicitly enumerating them. As a result, the group-cutset formulation is particularly suited for large-scale
instances, where explicit multi-commodity flow formulations are computationally
infeasible and compact formulations yield poor lower bounds.
To use BnC, we need to introduce several additional algorithmic building blocks which we now describe.

\ignore{
\rewrite{As $\mathcal{R}$ grows exponentially with the graph size, explicitly enumerating the group-cutset formulation would render the solver intractable. However, given a candidate solution, constraints~\eqref{eq:group-cutset} can be efficiently \emph{verified}, and for infeasible candidates, violated constraints can be identified using efficient flow-based algorithms. This property enables the development of an effective Branch-and-Cut solver, described in the following section.}
\former{Unfortunately, $\mathcal{R}$ contains an exponential number of subsets, which makes explicit representation of group-cutset formulation intractable.  Thus, we leverage the \emph{Branch-and-Cut} (BnC) approach which allows for a \emph{lazy} evaluation of the constraints \eqref{eq:group-cutset}. This requires the development of specialized algorithmic building blocks that allow an effective invocation for \gip problem. \adir{It sounds like we have missfortunetly designed a very bad formulation and now we develop heavy algorithms to compensate it. It is this way by design.}}
}

\section{Branch-and-Cut Solver for Group-Cutset formulation} \label{sec:alg}
The Branch-and-Cut (BnC) framework extends Branch-and-Bound by incorporating \emph{cutting planes}, i.e., additional constraints that are generated dynamically during the search.
{In the context of the group-cutset model introduced in Sec.~\ref{gip-gcutset}, the BnC framework allows for a \emph{lazy-constraint} evaluation. 
Specifically, the solver starts from a static partial formulation, and extends the formulation only with constraints that are relevant to the regions of the solution space explored by the solver, thus avoiding the overhead of enforcing constraints that are never active.} 
Such dynamic constraint enforcement is performed by a \emph{separation oracle}, which either identifies violated constraints or certifies that none exist. This allows BnC to operate on a compact formulation while retaining the strength of a much richer, potentially exponential constraint set. We present an oracle tailored to the group-cutset formulation in Sec. \ref{sep-oracle-design}. 

A complementary component of the BnC framework is a \emph{primal heuristic},
which transforms fractional LP relaxations into integer-feasible solutions
that serve as upper bounds in the search process.
While primal heuristics are standard in BnB-based MILP solvers, their
problem-specific design is especially important in lazy-constraint
formulations.
In such settings, generic heuristics employed by solvers such as
Gurobi~\cite{gurobi} may produce solutions that satisfy the current formulation
but violate constraints introduced later in the search.
We therefore introduce a problem-specific primal heuristic that explicitly
exploits the structure of \gip to provide valid solutions (Sec.~\ref{sec:heuristic-summary}).

\subsection{Group-Cutset Separation Oracles}
\label{sep-oracle-design}
During the BnC search, partial problem formulations are LP-relaxed and solved, providing solution candidates $\{x^c\}$. The separation oracle, given such candidate, determines whether one of the group-cutset constraints~\eqref{eq:group-cutset} is violated, i.e., there exists a set $R \in \mathcal{R}$ as defined in Sec.~\ref{gip-gcutset}, such that
$
\sum_{(u,v)\in \delta^{+}(R)} x^c_{uv} < 1.
$
If such a set exists, the corresponding constraint  is returned as a separating cut. We introduce two complementary separation oracles, along with a third hybrid oracle that combines their respective strengths.

\niceparagraph{Connectivity-based separation oracle.}
We first consider a fast separation oracle for verifying integral candidate solutions~$\{x^c\}$.
We construct a subgraph~$G^{c}=(V,E^{c})$ containing only the edges $(u,v)\in E$ for which $x^c_{uv}=1$. We then compute the strongly connected component (SCC) $R^c \subseteq V$ of the root~$r$ in $G^c$, i.e., $R^c$ contains any vertex $v\in V$ such that there exists a directed path from $r$ to $v$, and $v$ to $r$. For each group $S_i \in \mathcal{S}$, we check whether $R^c \cap S_i$ is empty.
If so, the set $R^c$ defines a violated group-cutset constraint, which is returned as a certificate. Otherwise, the oracle certifies that $\{x^c\}$ satisfies all group-cutset constraints.
As the complexity of the SCC operation is linear in the number of edges~\cite{cormen2009introduction} this oracle runs in
$O( |E|+|\mathcal{S}|\cdot|V|)$ time and guarantees that no invalid integral solution is ever accepted by the solver. Although this oracle is computationally efficient, it generates only a single violated constraint per integer candidate solution, which may make the  BnC solver struggle to develop meaningful problem representation.

\niceparagraph{Flow-based separation oracle.}
To strengthen constraint generation, we introduce a separation oracle applicable to both fractional and integral solutions, based on solving an $s$--$t$ max-flow problem (Alg.~\ref{alg:fb-check}).
The key idea is to interpret the fractional values $x^c_{uv}$ as capacities on directed edges and test whether the resulting network can route at least one unit of flow from the root $r$ to a vertex covering a given group $S_i$.
Under integrality of~$x^c_{uv}$, such a flow corresponds to a path composed of selected edges.
Formally, for each group $S_i \in \mathcal{S}$, we compute the maximum flow from $r$ to $S_i$, as illustrated in Fig.~\ref{fig:flow-interp}.
If the flow value is smaller than one, then no feasible solution can connect $r$ to $S_i$ using the current edge selections.
By the max-flow min-cut theorem~\cite{dantzig2003max}, this implies the existence of a cut separating $r$ from $S_i$ with total capacity less than one, yielding a violated group-cutset constraint that can be added as a separating inequality.
While effective at tightening LP relaxations, this oracle is computationally expensive, as fully verifying the validity of a candidate solution requires solving a max-flow problem for each group.
Using standard algorithms, Alg.~\ref{alg:fb-check} runs in $O(|\mathcal{S}|\cdot |V|^2\cdot |E|)$ time~\cite{ahuja1994network,edmonds1972theoretical}, motivating its use only in a limited manner within a hybrid approach.

{\small 
\begin{algorithm}[t]
\caption{Flow-based separation oracle ($\{x^c\}$, $G=(V,E)$, $\mathcal{S}$, $r$)}
\label{alg:fb-check}
\begin{algorithmic}[1]

\State Initialize an empty set $\mathcal{C}$ of cuts
\For{each group $S_i \in \mathcal{S}$}

    \State Define flow network $\mathcal{N}_i=(V_i, E_i, \kappa)$:
    \State \hspace{\algorithmicindent} $V_i=V \cup \{t\}$ \Comment{$t$ is an auxiliary sink vertex for $S_i$}
    \State \hspace{\algorithmicindent} Generate auxiliary edges $A_i \gets \{(v,t): v\in S_i\}$ with $\kappa(v,t)\gets 1$. 
    \State \hspace{\algorithmicindent} $E_i=E \cup A_i$.
    \State \hspace{\algorithmicindent} Update capacities for $E$ edges  $\kappa(u,v)\gets x^c_{uv}$ for any $(u,v)\in E$. 
    \State Compute a minimum $r$--$t$ cut in $\mathcal{N}_i$: $(R_i, (V\cup\{t\})\setminus R_i)$.
    \If{$\sum_{(u,v)\in\delta^{+}(R_i)} x^c_{uv} < 1$}
        \State Add the violated constraint $\sum_{(u,v)\in\delta^{+}(R_i)} x_{uv} \ge 1$ to $\mathcal{C}$.
    \EndIf
\EndFor

\State \Return $\mathcal{C}$
\end{algorithmic}
\end{algorithm}}

\niceparagraph{Combined separation oracle.}
The two prior oracles offer complementary strengths.
The connectivity-based oracle is computationally efficient and guarantees
correctness by validating integral candidate solutions. However, it can generate
at most a single separating constraint per integral solution, which limits its
ability to substantially tighten the relaxation.
In contrast, the flow-based oracle is capable of generating multiple effective
separating constraints, including at fractional solutions, but its computational cost prevents it from being used to validate or rule out candidate solutions.
We therefore propose a combined separation oracle that applies connectivity-based checks to all integer-feasible candidates and selectively applies flow-based separation to a \emph{uniformly sampled subset} of groups at fractional solutions.
This balances relaxation strength and computational cost while preserving correctness.
Importantly, sampling introduces no loss of correctness: although violated constraints for unsampled groups may be missed at fractional solutions, any integral solution is ultimately validated by the connectivity oracle. As shown in Sec.~\ref{sec:oracle-eval}, this combined approach yields strong cuts at low computational cost, significantly improving convergence while maintaining correctness guarantees. 
Furthermore, as we will demonstrate empirically, the method is insensitive  to the exact value of the group sample size.

\subsection{Primal Heuristic for \gip}
\label{sec:heuristic-summary}

We present a problem-specific primal heuristic for \gip.
While applicable to all SEC variants (Sec.~\ref{sec:milp-for-gip}), it is particularly important for lazy-constraint formulations.
In such settings, generic MILP heuristics (e.g.~\cite{bertacco2007feasibility,danna2005exploring,fischetti2005feasibility}) may generate solutions that satisfy only the currently enforced constraints, yet violate constraints that will be introduced later by the separation oracle.
This issue is most pronounced in the early stages of the search, when the formulation is still highly partial, making incumbent generation unreliable and motivating a heuristic that is explicitly aware of the full \gip structure.

Our heuristic is tightly integrated with the BnC search and exploits information from the current LP relaxation.
At a high level, it consists of three phases.
First, edge costs are modified using the fractional LP solution $\{x_e\}_{e \in E}$ by defining
$c_n(e) := c(e)\cdot(1 - x_e)$,
which biases the search toward edges favored by the relaxation.
Second, a group-covering tree rooted at the start vertex is greedily constructed in the discounted graph, incrementally connecting the root to a representative vertex of an uncovered group while ensuring connectivity.
Third, the resulting tree is augmented and traversed to produce a valid \gip tour.
Rather than doubling tree edges, we add a minimum-weight matching over odd-degree tree vertices, yielding an Eulerian subgraph whose traversal produces a lower-cost tour~\cite{christofides2022worst}.
This heuristic reliably generates high-quality incumbent solutions early in the BnC search, leading to substantially improved upper bounds and faster convergence, and may be considered as a \emph{problem-aware rounding mechanism} applied on fractional LP solutions.
A complete algorithmic description is provided in App.~\ref{appendix:gip-heuristic}, and an empirical evaluation appears in App.~\ref{heuristic-eval}.

\ignore{
\section{Scalable MILP Formulation using Lazy Constraints} \label{sec:cutset}
\oren{motivation\\}
The multi-commodity flow formulation demonstrates that enforcing explicit root-to-group connectivity yields tight LP relaxations, but its explicit
representation is prohibitively large for realistic problem instances. We therefore seek a formulation that retains this connectivity strength without
introducing an explicit flow commodity for each group.

\oren{approach\\}
In this section, we present a new subtour-elimination formulation for \gip based on \emph{group cutsets}. This formulation enforces root-to-group connectivity
through an exponential family of cut constraints, yielding a strong MILP relaxation without an explicit multi-commodity representation.
Since these constraints are too numerous to be instantiated explicitly, we solve the resulting model
using a \emph{Branch-and-Cut} (BnC) approach~\cite{padberg1991branch,
mitchell2002branch}, in which violated constraints are generated on demand during the optimization.

\oren{Useage\\}
At a high level, BnC alternates between solving relaxations of the MILP and incrementally strengthening them by adding violated constraints.
Rather than enforcing all group-cutset constraints upfront, the solver starts without them, and iteratively identifies and inserts only those constraints that are necessary to rule out invalid, subtour-based solutions.
This approach retains the connectivity strength of flow-based formulations such as MCF while avoiding their prohibitive size.

\oren{alg components\\}
The resulting solver is built around two key algorithmic components. First, a \emph{separation oracle} examines candidate solutions and identifies violated group-cutset constraints when they exist. Second, a problem-specific \emph{primal heuristic}  is invoked to efficiently construct feasible inspection tours, enabling rapid improvement of upper bounds during the search. Together, these components enable an efficient and scalable BnC solver for \gip that preserves the strength of flow-based formulations while remaining practical for large-scale inspection planning problems. \kiril{it would be worth briefly mentioning, potentially in the appendix, when those components are invoked. }

\subsection{Group-Cutset Subtour Elimination Constraints}
\label{gip-gcutset}
To develop a lazy formulation for \gip, we adapt classical cutset-based subtour-elimination constraints from TSP~\cite{applegate2001tsp} to the group-covering setting.
Specifically, we consider vertex cuts that separate the root from an entire group.
Formally, for any group $S_i \in \mathcal{S}$, we examine all partitions of the vertex set $V$ into two disjoint subsets $R \subset V$ and $V \setminus R$ such that the root $r \in R$ and $R \cap S_i = \emptyset$.
Such a partition defines a cut between $R$ and its complement.
If a tour were to remain entirely within $R$, it would be impossible to visit any vertex in $S_i$, and thus the corresponding POI would not be inspected.

Therefore, any feasible \gip tour must include at least one edge crossing this cut, connecting a vertex in $R$ to a vertex in $V \setminus R$.
By requiring the selected edges to cross every such root-group cut, we enforce
connectivity between the root and each group.
As we show below, these constraints force the formation of a single tour visiting $r$.
Formally, let $\mathcal{R}$ denote the family of vertex subsets that contain the root but exclude at least one group:
\begin{equation}
\mathcal{R} := \left\{ R \subseteq V \;\middle|\; r \in R \;\text{and}\; \exists S_i \in \mathcal{S}
\text{ such that } R \cap S_i = \emptyset \right\}.
\end{equation}
For any vertex subset $R \subseteq V$, we denote its outgoing edge set as
\[
\delta^{+}(R) := \{(u,v) \in E \mid u \in R,\; v \notin R\}.
\]
\adir{V or $\mathcal{V}$}

The baseline formulation~\eqref{eq:gtsp-milp-obj}--\eqref{eq:gtsp-binary} is then
augmented with the following family of group-cutset constraints:
\begin{equation}
\label{eq:group-cutset}
\sum_{(u,v)\in \delta^{+}(R)} x_{uv} \ge 1, \qquad \forall R \in \mathcal{R}.
\end{equation}
Namely, for any vertex set $R$ that contains the root $r$ but does not intersect
at least one group $S_i \in \mathcal{S}$, at least one selected edge must leave
$R$.
These constraints enforce root-group connectivity, but do not explicitly impose
that the selected edges form a single tour: in principle, multiple disconnected
components may still satisfy~\eqref{eq:group-cutset}.
The following lemma shows that, at optimality, such disconnected components
cannot occur.

\begin{lemma}[Group-cutset constraints eliminate subtours]
\label{lem:group_cutset_sec}
Assume $c(e) > 0$, for all $e \in E$.\footnote{
If zero-cost edges are allowed, i.e., $c(e)\ge 0$, an optimal solution may
contain additional disconnected components of zero total cost.
Such components can be removed without affecting feasibility or optimality,
yielding an optimal solution that satisfies the lemma.
}
Consider the baseline \gip MILP~\eqref{eq:gtsp-milp-obj}-\eqref{eq:gtsp-binary}
augmented with the group-cutset constraints~\eqref{eq:group-cutset}.
Then any optimal integral solution forms a single
connected component containing the root $r$. \kiril{Shouldn't we care about tours, not CCs?}
\end{lemma}

\begin{proof}
Let \adir{Notice notations - $\{\}$ when needed}$x^*$ be an optimal integral solution and let
$F := \{e \in E \mid x^*_e = 1\}$ be its selected directed edges.
Let $C_r \subseteq V$ denote the (undirected) connected component of $r$ in the
subgraph induced by $F$, and let $F_r \subseteq F$ be the edges incident only to
vertices in $C_r$.
Define a new solution $\hat{x}$ by setting $\hat{x}_e = 1$ if $e \in F_r$ and
$\hat{x}_e = 0$ otherwise. We show that $\hat{x}$ is feasible.
\adir{From here on - necessary?}
Constraint~\eqref{eq:gtsp-root} holds since $r \in C_r$ and therefore retains at
least one outgoing selected edge.
For each group $S_i \in \mathcal{S}$, the group-cutset constraints~\eqref{eq:group-cutset} imply that
$S_i$ is reachable from $r$ in the subgraph induced by $F$, and hence
$S_i \cap C_r \neq \emptyset$.
Thus, $\hat{x}$ satisfies the group-coverage constraints
\eqref{eq:gtsp-groups}.
Flow-balance constraints~\eqref{eq:gtsp-balance} remain satisfied, as removing
edges outside $C_r$ does not affect balance at vertices within $C_r$.
Finally, $\hat{x}$ is integral by construction.

If $F \neq F_r$, then $\hat{x}$ removes at least one selected edge.
Since all edge costs are strictly positive, this strictly decreases the objective
value, contradicting the optimality of $x^*$.
Therefore, $F = F_r$, and all selected edges lie in the connected component of
$r$.
\end{proof}

\begin{corollary}
\label{thm:MILP}
Optimal solutions to the MILP defined by
\eqref{eq:gtsp-milp-obj}-\eqref{eq:gtsp-binary} augmented with
\eqref{eq:group-cutset} correspond to optimal solutions of \gip.
\end{corollary}

\subsection{Group-Cutset Separation Oracles}
\label{sep-oracle-design}
While the MILP~\eqref{eq:gtsp-milp-obj}-\eqref{eq:gtsp-binary} augmented with group-cutset constraints~\eqref{eq:group-cutset} accurately models \gip (Cor.~\ref{thm:MILP}), the number of such constraints is exponential in $|V|$. Explicitly enumerating all group-cutset constraints therefore yields an intractably large MILP.
We address this challenge using a lazy-constraint branch-and-cut framework,\adir{clarify BnC, BnB, lazy constraints.} in which group-cutset constraints are generated on demand.
Specifically, whenever a candidate solution violates an inequality of the form~\eqref{eq:group-cutset}, the corresponding constraint is added to the model.
This process is repeated throughout the optimization and thus requires efficient separation oracles.

Formally, given a candidate solution $x^c$, we seek to determine whether there exists a set $R \subset V$ with $r \in R$ and $R \cap S_i = \emptyset$ for some group $S_i \in \mathcal{S}$ such that
$
\sum_{(u,v)\in \delta^{+}(R)} x^c_{uv} < 1.
$
If such a set exists, the corresponding group-cutset constraint is violated and can be returned as a separating cut.

\adir{Still not as good as it can be - try to emphasis the conceptual step between oracle 1 and 2 alone, to the combined oracle.}
We employ two complementary separation oracles: a fast oracle for integral candidate solutions, based on graph connectivity, and a stronger oracle for fractional solutions, based on network flow. We begin with the connectivity oracle, which serves as the correctness validation backbone of our BnC scheme.

\paragraph{Connectivity oracle.}
When the candidate solution $x^c$ is integral, we apply a separation oracle based on graph connectivity (Alg.~\ref{alg:cc-check}).
We first construct a subgraph $G^c=(V,E^c)$ of the original graph $G$, 
where vertices $u$ and $v$ are connected if at least one of the directed edges $(u,v)$ or $(v,u)$ is selected.

We then compute the connected component $R^c \subseteq V$ of the root $r$ in $G^c$ (Line~2). For each group $S_i \in \mathcal{S}$, we check whether $R^c \cap S_i$ is empty (Lines~3-4).
If so, the set $R^c$ defines a violated group-cutset constraint, which is returned as a certificate (Line~5).
If all groups intersect $R^c$, the oracle certifies that $x^c$ satisfies all group-cutset constraints (Line~6).

The connectivity oracle runs in
$O(|\mathcal{S}|\cdot|V| + |E|)$ time \adir{verify} and guarantees that no invalid integer-feasible solution is ever accepted by the solver.

\begin{algorithm}[t]
    \caption{Connectivity-based separation oracle \adir{polish}}
    \label{alg:cc-check}
    \begin{algorithmic}[1]
        \Ensure $\{x^c\}$; $G = (V,E)$; $\mathcal{S}$
        \Require $\forall e \in E:$ $x^c(e) \in \{0,1\}$
        
        \State Set $G_c := (V, E_c)$, where vertices $u$ and $v$ are connected by an undirected edge in $E_c$ if at least one of the directed edges $(u,v)$ or $(v,u)$ satisfies $x^c_{uv} = 1$.

        \State Set $R^c=\{v \in V:\exists \text{ path between $v$ and $r$ in $G^c$}\} $.
        \For{\textbf{each} $S_i \in \mathcal{S}$}
            \If {$R^c\cap S_i = \emptyset$}
                \State \Return constraint $\sum_{(u,v)\in\delta^+(R^c)} x_{uv} \ge 1$
            \EndIf
        \EndFor
        
        \State \Return \texttt{Valid Solution}
        
    \end{algorithmic}
\end{algorithm}

\paragraph{Flow oracle.}
When the candidate solution $x^c$ contains fractional values, we apply a separation oracle based on solving an $s$-$t$ max-flow problem (Alg.~\ref{alg:fb-check}).
Unlike the connectivity oracle, the flow oracle is not used to validate feasibility, but rather to generate strong cutting planes that tighten the LP relaxation.
The key idea is to interpret the fractional values $x^c_{uv}$ as capacities on the directed edges of the graph, and to test whether this capacitated network
can support a flow of at least one unit from the root $r$ to a vertex covering a given group $S_i$.
The threshold value $1$ is chosen to match the semantics of the group-cutset constraints~\eqref{eq:group-cutset}, which require that at least one selected edge crosses any root-group cut.
Formally, for each group $S_i \in \mathcal{S}$, we construct a flow network and compute the maximum flow from $r$ to $S_i$.
If the value of this flow is strictly smaller than $1$, then no feasible integral solution can route a path from $r$ to that group using the currently selected edges.
By the max-flow min-cut theorem~\cite{dantzig2003max}, this implies the existence of a vertex cut $R$ separating $r$ from $S_i$ whose total capacity is less than $1$.
Such a cut corresponds directly to a violated group-cutset constraint, which can be added to the MILP as a separating inequality.

Compared to the connectivity oracle, the flow oracle can identify violated constraints already at fractional solutions, and therefore produces substantially stronger cuts. \adir{avoid this terminology..}
However, solving a max-flow problem for each group is computationally expensive, which limits the practicality of using this oracle exhaustively.

Note that Alg.~\ref{alg:fb-check} can be computed in $O(\vert \mathcal{S}\vert \cdot  \vert V\vert^2 \cdot \vert E \vert)$ time using the shortest augmenting path max-flow algorithm \cite{ahuja1994network, edmonds1972theoretical}.

{\small 
\begin{algorithm}[t]
\caption{Flow-based separation oracle ($\{x^c\}$, $G=(V,E)$, $\mathcal{S}$, $r$)}
\label{alg:fb-check}
\begin{algorithmic}[1]

\State Initialize an empty set $\mathcal{C}$ of cuts
\For{each group $S_i \in \mathcal{S}$}

    \State Define flow network $\mathcal{N}_i=(V_i, E_i, \kappa)$:
    \State \hspace{\algorithmicindent} $V_i=V \cup \{t\}$ \Comment{$t$ is an auxiliary sink vertex for $S_i$}
    \State \hspace{\algorithmicindent} Generate auxiliary edges $A_i \gets \{(v,t): v\in S_i\}$ with $\kappa(v,t)\gets 1$. 
    \State \hspace{\algorithmicindent} $E_i=E \cup A_i$.
    \State \hspace{\algorithmicindent} Update capacities for $E$ edges  $\kappa(u,v)\gets x^c_{uv}$ for any $(u,v)\in E$. 
    \State Compute a minimum $r$--$t$ cut in $\mathcal{N}_i$: $(R_i, (V\cup\{t\})\setminus R_i)$.
    \If{$\sum_{(u,v)\in\delta^{+}(R_i)} x^c_{uv} < 1$}
        \State Add the violated constraint $\sum_{(u,v)\in\delta^{+}(R_i)} x_{uv} \ge 1$ to $\mathcal{C}$.
    \EndIf
\EndFor

\State \Return $\mathcal{C}$
\end{algorithmic}
\end{algorithm}}

\paragraph{Combined separation oracle.}
The above oracles offer complementary advantages.
The connectivity oracle is extremely efficient and guarantees correctness by certifying all integer-feasible solutions, but it can only generate cuts from integral structures and therefore tends to produce relatively weak
constraints.
The flow oracle, in contrast, is capable of generating strong cuts \adir{avoid this terminology..} from fractional solutions, but is too computationally expensive when applied to all
groups.

We therefore combine the two oracles into a single separation strategy that retains correctness while enabling scalable strengthening of the LP relaxation.
In the combined oracle, the connectivity oracle is used to validate all integer-feasible candidate solutions, ensuring that no invalid inspection plan is ever accepted.
The flow oracle is used selectively to strengthen the formulation by generating high-quality cuts from fractional solutions.

To make the flow-based separation scalable, we apply it only to a \emph{uniformly sampled subset of groups} $\mathcal{S}_{\rm sample} \subset \mathcal{S}$.
This sampling introduces no loss of correctness: although the flow oracle may miss violated constraints for unsampled groups, any integer-feasible solution
is ultimately validated by the connectivity oracle.
In practice, this approach yields a small number of strong cuts at low computational cost, significantly improving convergence while maintaining theoretical correctness. In Sec.~\ref{oracle-eval}, we empirically demonstrate that the combined oracle substantially outperforms either oracle in isolation.

\subsection{Inspection Plan Heuristic}
\label{sec:heuristic-summary}
To efficiently generate high-quality incumbent solutions within the Branch-and-Cut (BnC) framework, we employ a problem-specific primal heuristic tailored to \gip. Such a heuristic is particularly important in lazy-constraint formulations: while fully instantiated MILPs
enable modern solvers such as Gurobi~\cite{gurobi} to exploit powerful generic heuristics~\cite{bertacco2007feasibility,fischetti2005feasibility,danna2005exploring}, lazy-constraint models operate on incomplete formulations. As a result, off-the-shelf solver
heuristics may produce solutions that satisfy the currently active constraints but violate those introduced later in the search.
Our heuristic is designed to integrate tightly with the BnC process and to exploit information revealed by the current LP relaxation. At a high level, it
proceeds in three conceptual phases.

First, the heuristic leverages the fractional LP solution by modifying edge costs according to their fractional values. Edges that appear with high fractional weight in the LP relaxation are discounted, biasing subsequent computations toward structures already favored by the relaxation. This step can be viewed as a problem-aware LP-rounding mechanism that guides the heuristic toward promising regions of the solution space.
Second, using the discounted graph, the heuristic constructs a group-covering tree rooted at the start configuration. Inspired by Steiner Tree heuristics, the tree incrementally connects the root to at least one vertex from each POI group, ensuring full coverage while maintaining a compact structure. This phase enforces coverage and connectivity, but does not yet impose tour structure.
Finally, the group-covering tree is transformed into a valid \gip tour.
The tree is augmented to become Eulerian, either via minimum-weight perfect matching or a faster greedy matching strategy, and an Eulerian traversal is extracted. The resulting tour is a feasible inspection plan that is consistent with the solver’s current relaxation.

Together, these steps enable the heuristic to produce high-quality incumbent solutions early in the BnC search, substantially improving upper bounds and accelerating convergence. A complete algorithmic description of the heuristic, along with an empirical evaluation demonstrating its advantages over generic solver heuristics and prior approaches, is provided in App.~\ref{gip-heuristic}.
}

\section{Experimental Results}
\label{sec:evaluation}
We empirically evaluate how different MILP formulations and their associated algorithmic components for \gip trade off scalability and solution-quality guarantees.
We report two quantities that BnC solvers maintain  throughout the search: an upper bound $c_{\rm UB}$, given by the cost of the best \emph{incumbent} feasible solution, and a lower bound $c_{\rm LB}$, obtained from LP relaxations of the MILP. 
Another measure we report is the \emph{optimality gap}, defined as $\mathrm{Gap}:= 100 \cdot 
(c_{\rm{UB}} - c_{\rm{LB}}) / c_{\rm{UB}}$, 
which quantifies the practical tightness of the solver’s near-optimality certification. This measure is particularly important given the provable hardness of approximating \gip (App.~\ref{sec:Approx-bounds}).


\ignore{
solver performance 
is best assessed not only by the quality of the \emph{incumbent} feasible solution (upper bounds), but also by the tightness of provable lower bounds.
Accordingly, we emphasize the \emph{optimality gap}
as the principal performance metric throughout our experiments, defined as
\begin{equation}
    \mathrm{Gap}\% = 100 \cdot \frac{\mathrm{Incumbent} - \mathrm{LowerBound}}{\mathrm{Incumbent}}.
\end{equation}
}

We conduct our evaluation by progressing from medium-scale real-world benchmarks to larger-scale controlled simulated experiments, and finally to
targeted ablation studies, in order to examine the effects of formulation and solver design.
Across all settings, the proposed \algname{Group-Cutset} formulation consistently produces stronger lower bounds than the other tested formulations, and exhibits slower performance degradation on larger instances.
We complement these results in App.~\ref{appendix:complement-eval} with an
evaluation of small-scale instances, where \algname{Group-Cutset} continues to perform competitively. 
The appendix further includes an ablation study of primal-heuristic design.


The experiments were conducted on a  laptop with an Intel Core Ultra~9 185H CPU and 64\,GB RAM, without GPU acceleration or explicit parallelization beyond that provided by the MILP solver (Gurobi~~\cite{gurobi}). Algorithms described in Sec.~\ref{sec:alg} are implemented in Python, interfacing with Gurobi’s \Cpp solver core via the \texttt{gurobipy} package. Our code is publicly available in the
\href{https://github.com/adirmorgan/GraphInspectionPlanning}{\gip  repository}.

\subsection{Evaluation Scenarios}
Our evaluation uses both real-world inspection datasets and controlled simulated scenarios, to study solver behavior across varying graph sizes and numbers of POIs.
The \texttt{CRISP} scenario~\cite{anderson2017continuum,mahoney2016reconfigurable}, adapted from~\cite{fu2019toward}, models medical inspection with a continuum robot tasked with inspecting 4{,}203 POIs in a confined anatomical cavity, yielding \gip instances with dense and highly overlapping coverage groups.
The \texttt{Bridge} scenario, adapted from~\cite{fu2021computationally}, considers aerial inspection of 3{,}346 POIs distributed across a large structure, requiring inspection plans that connect distant regions of the roadmap.
\texttt{Controlled} scenarios complement these benchmarks, enabling systematic evaluation across a wide range of problem scales using a simplified planar point-robot simulator (App.~\ref{appendix:simulator}).
Experiment scenarios are illustrated in Fig.~\ref{fig:eval-scenarios}.

\begin{figure}
    \centering
    \begin{subfigure}[b]{0.32\linewidth}
        \centering
        \includegraphics[width=\linewidth]{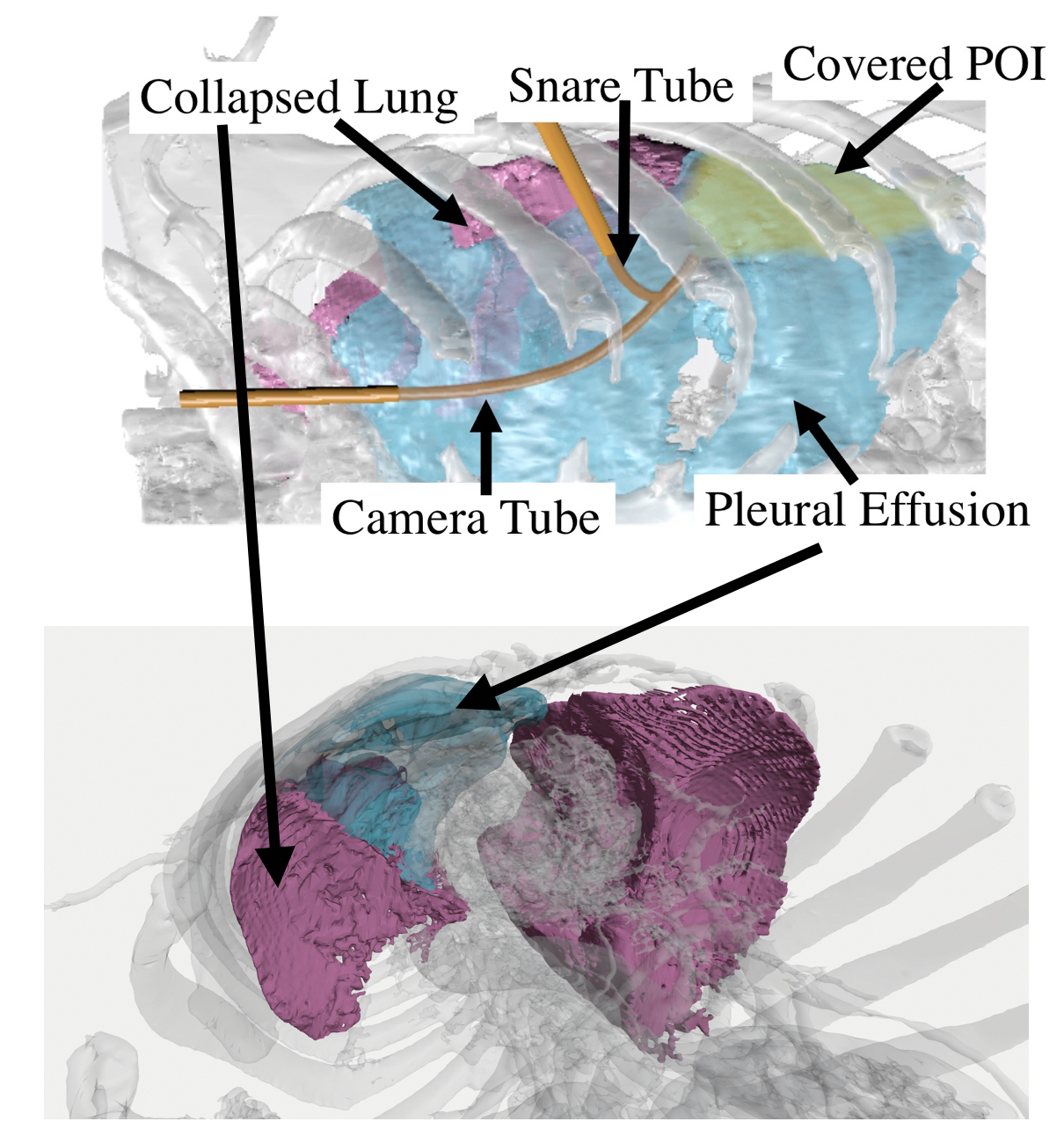}
        \caption{\texttt{CRISP}}
    \end{subfigure}\hfill
    \begin{subfigure}[b]{0.32\linewidth}
        \centering
        \includegraphics[height=4cm, width=\linewidth, keepaspectratio]{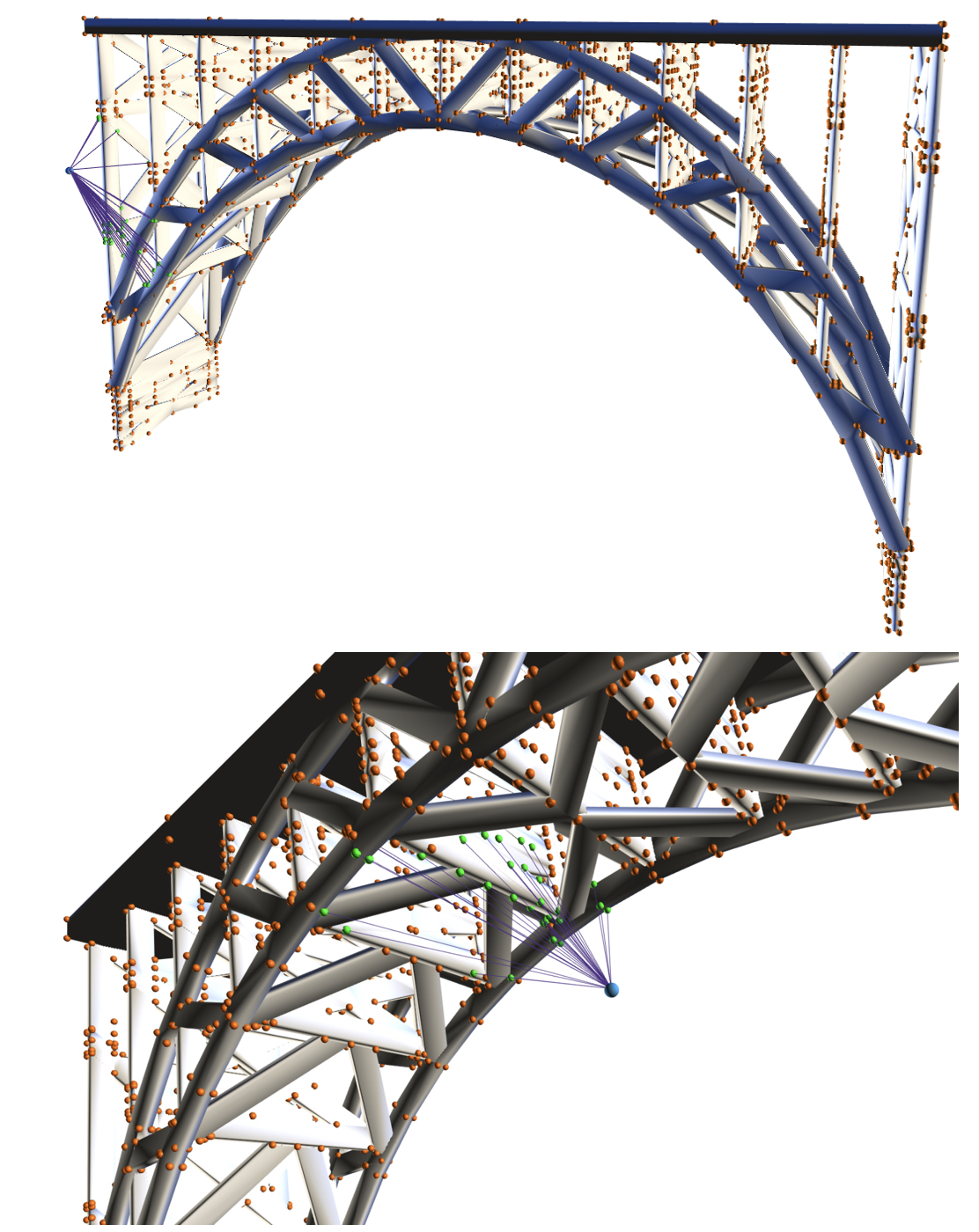}
        \caption{\texttt{Bridge}}
    \end{subfigure}\hfill
    \begin{subfigure}[b]{0.33\linewidth}
        \centering
        \includegraphics[height=5cm, width=\linewidth, keepaspectratio]{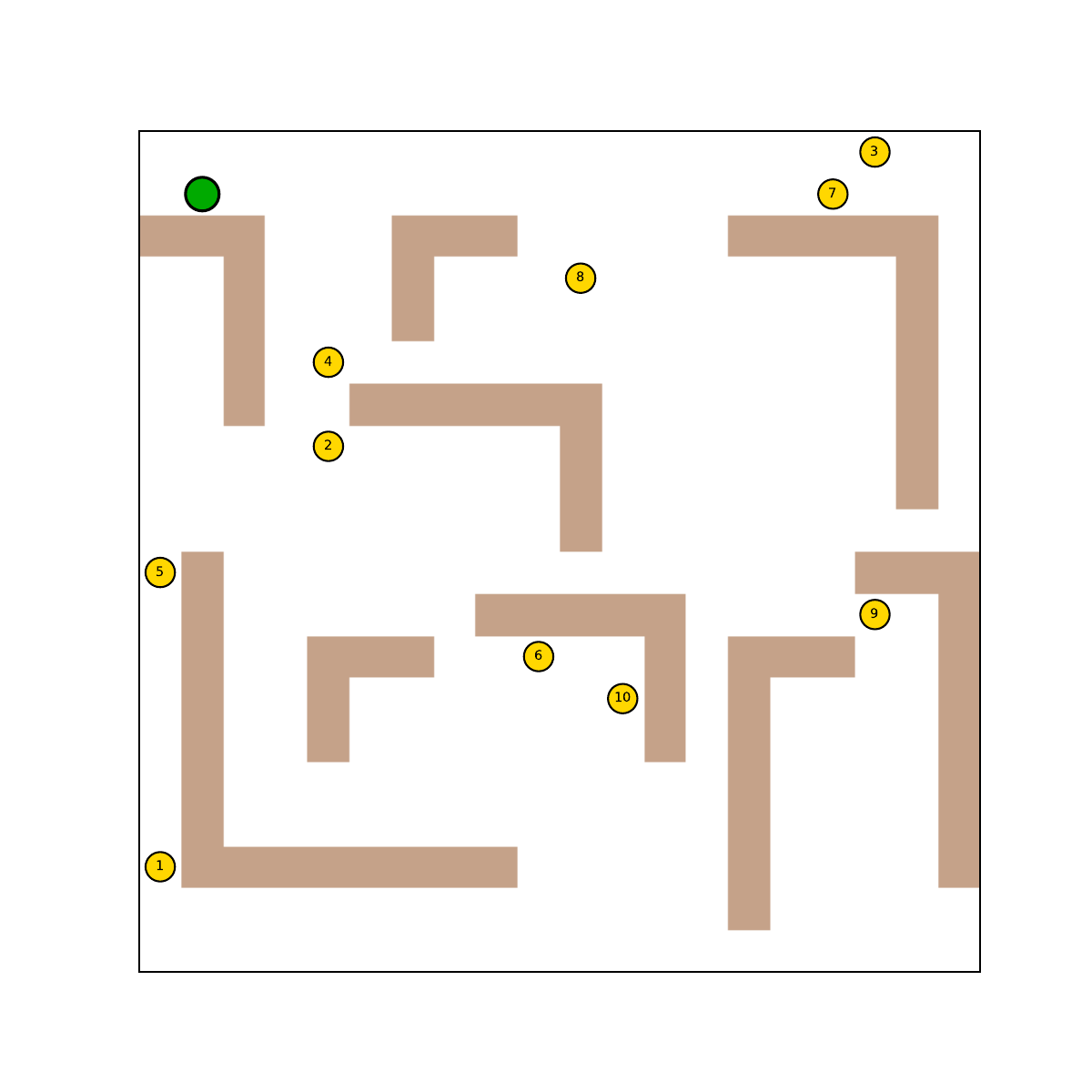}
        \caption{\texttt{Controlled}}
    \end{subfigure}
    
    \caption{Experimental evaluation scenarios. (a) CRISP robot, composed of a snare-tube and a camera-tube, inside the cavity of a patient's lungs. (b) Drone inspecting POIs marked on a bridge structure. (c) Planar point robot in a maze, starting at the green point and tasked with inspecting all POIs (yellow).}
    \label{fig:eval-scenarios}
\end{figure}

\subsection{Comparative Evaluation of MILP Formulations}\label{sec:form-comp}
We compare the performance of the different MILP formulations, i.e., the baseline MILP detailed in Sec.~\ref{gip-form} with the three different SEC detailed in Sec.~\ref{gip-sec-poly}, which we refer to as
\algname{SCF}, \algname{MCF} and \algname{Group-Cutset}
together with the charge-variables based formulation of \citet{mizutani2024leveraging}, which we refer to as \algname{Charge}.
The compact formulations (\algname{SCF}, \algname{Charge}) are solved within the BnB framework
while the \algname{Group-Cutset} formulation is solved using the BnC framework with the combined separation oracle using a group sample size of $100$, accompanied by our developed \gip primal heuristic, elaborated in App.~\ref{appendix:gip-heuristic}.
%
For each variant, we report $c_{\rm{UB}},c_{\rm{LB}}$, and the optimality gap $\rm{Gap}$ as a function of each algorithm's  running time. 
Due to its explicit large representation, \algname{MCF} formulation exceeds available memory at scales evaluated here, and is therefore omitted from these experiments and is discussed on App.~\ref{appendix:small-exp}.

\niceparagraph{Medium-size Instances.} We begin with medium-sized 
scenarios (\texttt{CRISP} and \texttt{Bridge}), with a time limit of $1{,}000$ seconds per instance.
Roadmaps are constructed using IRIS-CLI~\cite{fu2021computationally} with
$1{,}000$ and $2{,}000$ vertices, yielding graphs with over $20{,}000$ and
$40{,}000$ edges, respectively.
Results are shown in Fig.~\ref{fig:realResults-Both}.

\begin{figure}[!ht]
     \centering
     \begin{subfigure}[b]{1\linewidth}
         \centering
         \includegraphics[width=\textwidth]{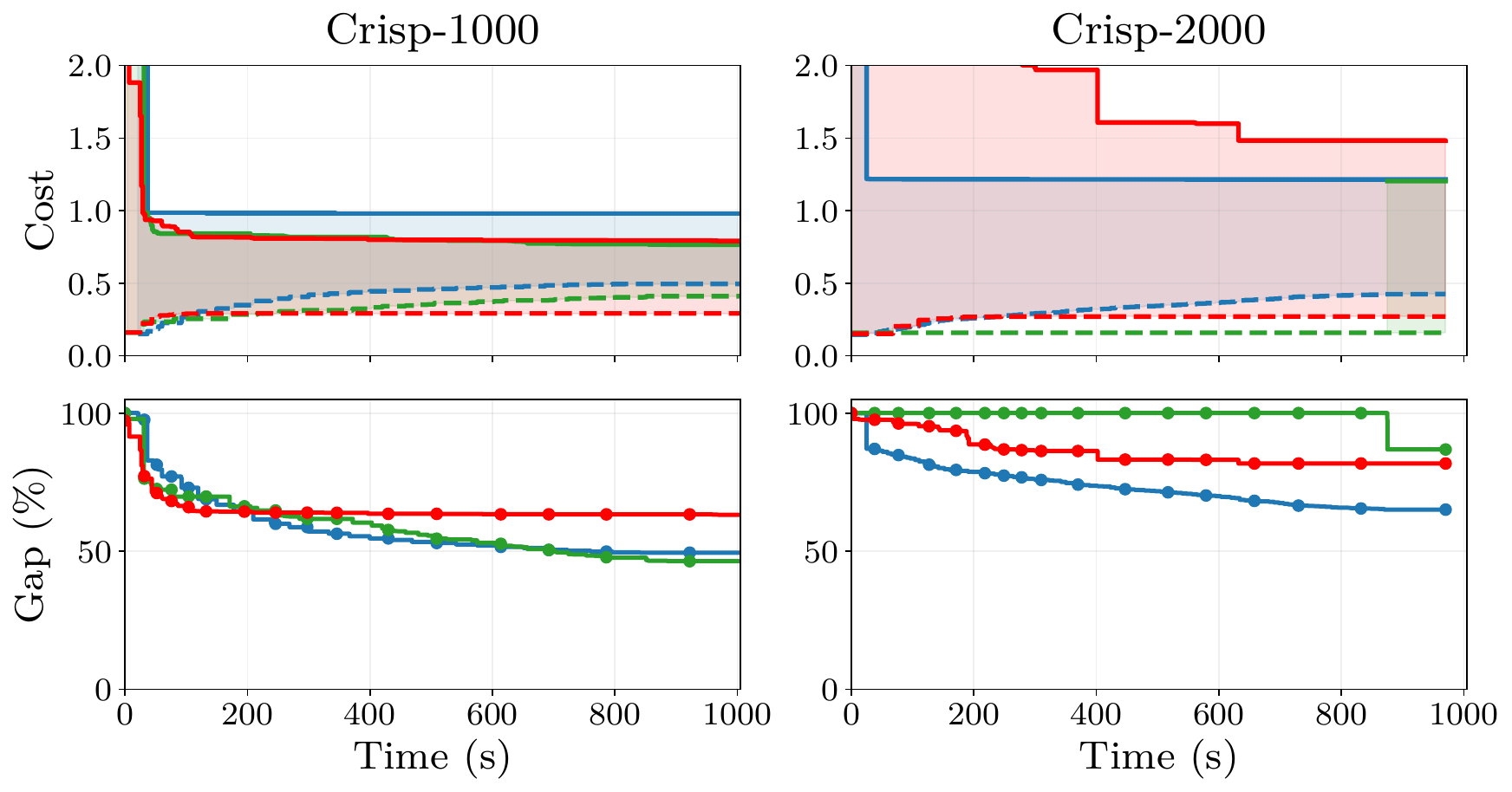}
         \label{fig:crisp}
     \end{subfigure}
     \vspace{-0.5cm} 
     
     \begin{subfigure}[b]{1\linewidth}
         \centering
         \includegraphics[width=\textwidth]{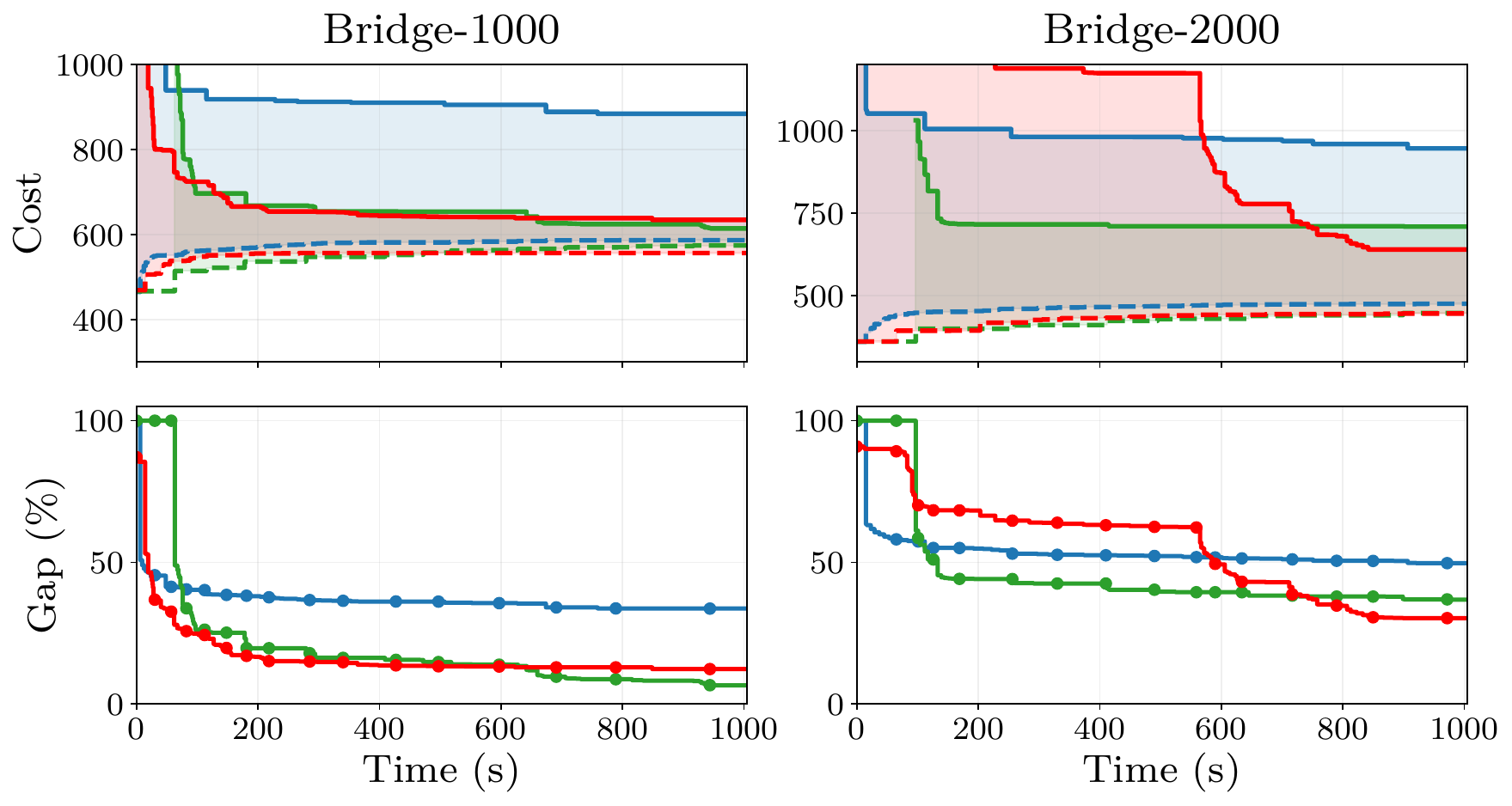}
         \label{fig:drone}
     \end{subfigure}
     \vspace{-0.4cm}
     
     {\centering
      \includegraphics[width=\textwidth]{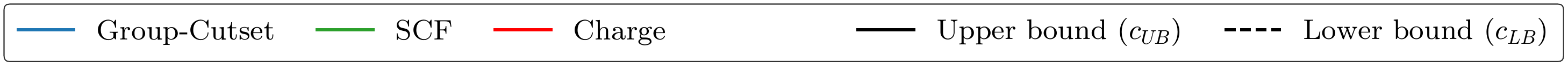}\par}
      
     \caption{Solvers evaluation on real-world instances.}
     \label{fig:realResults-Both}
\end{figure}

Across all evaluated instances, \algname{Group-Cutset} consistently produces tighter lower bounds than the compact formulations, an effect that is particularly pronounced in the \texttt{CRISP} scenarios, where its lower bounds continue to improve throughout the run while those of compact formulations stabilize early.
In contrast, compact formulations, most notably \algname{SCF}, tend to yield better incumbent solutions.
Compared to the prior MILP solver \algname{Charge}, \algname{SCF} achieves
better upper and lower bounds on both $1{,}000$-vertex instances, whereas
\algname{Charge} performs best on the \texttt{Bridge-2000} instance.

These results indicate instance-dependent behavior.
While the \algname{Group-Cutset} primal heuristic performs competitively on
\texttt{CRISP}, it is weaker on \texttt{Bridge}, suggesting limited generality of the current problem-specific heuristic (Sec.~\ref{sec:heuristic-summary}).
In contrast, the fully instantiated compact formulations allow Gurobi to exploit a broader set of generic primal heuristics, enabling better adaptation to different instance structures.
Overall, these complementary strengths suggest that hybrid approaches combining strong incumbent generation from compact formulations with the tighter lower bounds of \algname{Group-Cutset} may be a promising direction for future work.

\niceparagraph{Large Instances.}\label{sec:large-eval}
The medium-sized \gip instances correspond to relatively small roadmap sizes when compared to those encountered in realistic motion-planning problems~\cite{Panasoff.Solovey.25}. We therefore turn to simulated \gip instances with substantially larger graphs and numbers of POIs to explicitly stress solver scalability. Time limits are set to 500 seconds per instance.
Solver setups are identical to those used in the previous experiments, with the exception that the compact formulations (\algname{SCF}, \algname{Charge}) are augmented with the our \gip heuristic.\footnote{This heuristic was added because Gurobi’s internal heuristics struggled to produce feasible solutions within the time limits for the large instances considered. All reported results improved when this heuristic was included.}

\begin{figure}[t]
    \centering
    \begin{subfigure}[b]{0.24\linewidth}
        \centering
        \includegraphics[width=\linewidth]{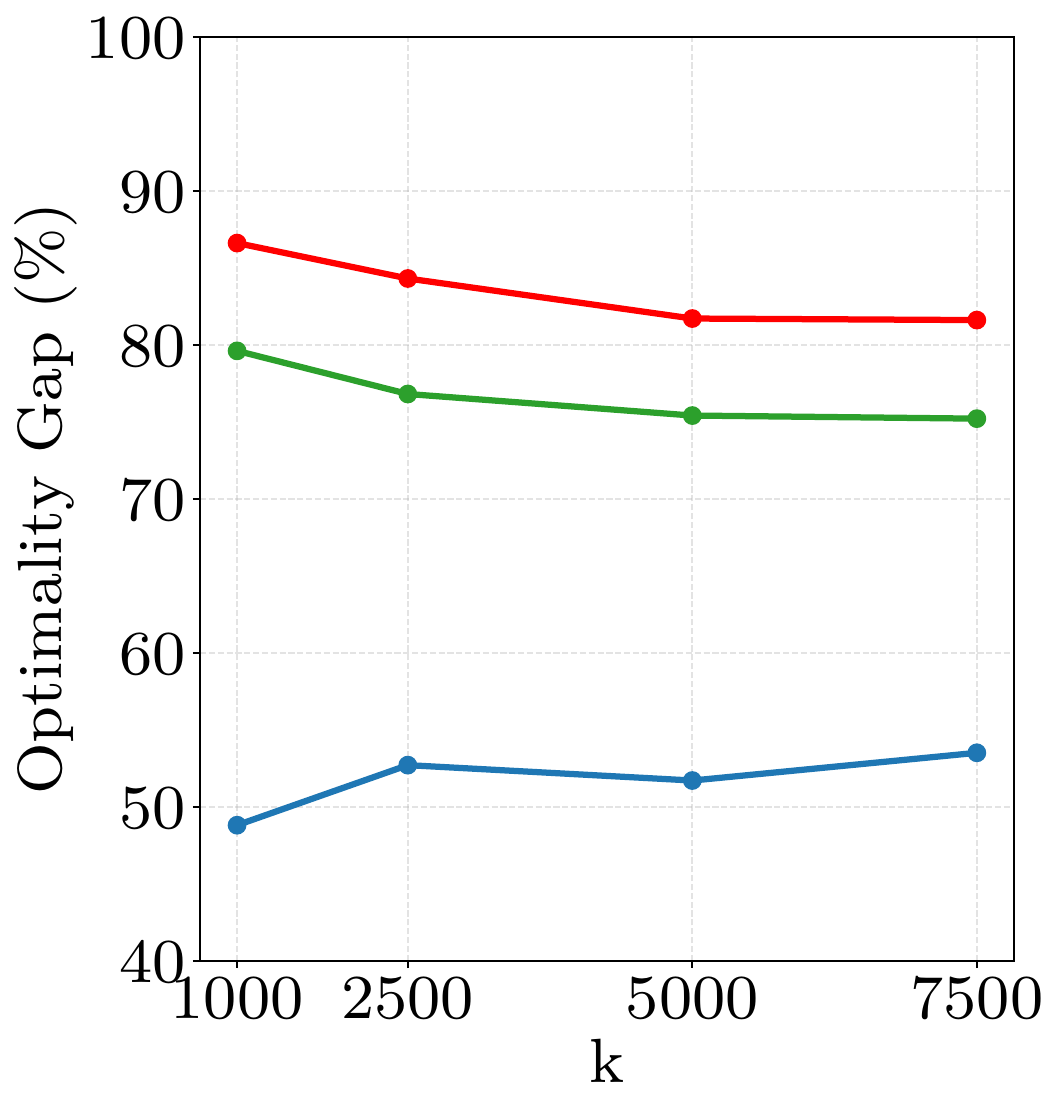}
        \caption{$n=3{,}500$}
    \end{subfigure}
    \hfill
    \begin{subfigure}[b]{0.24\linewidth}
        \centering
        \includegraphics[width=\linewidth]{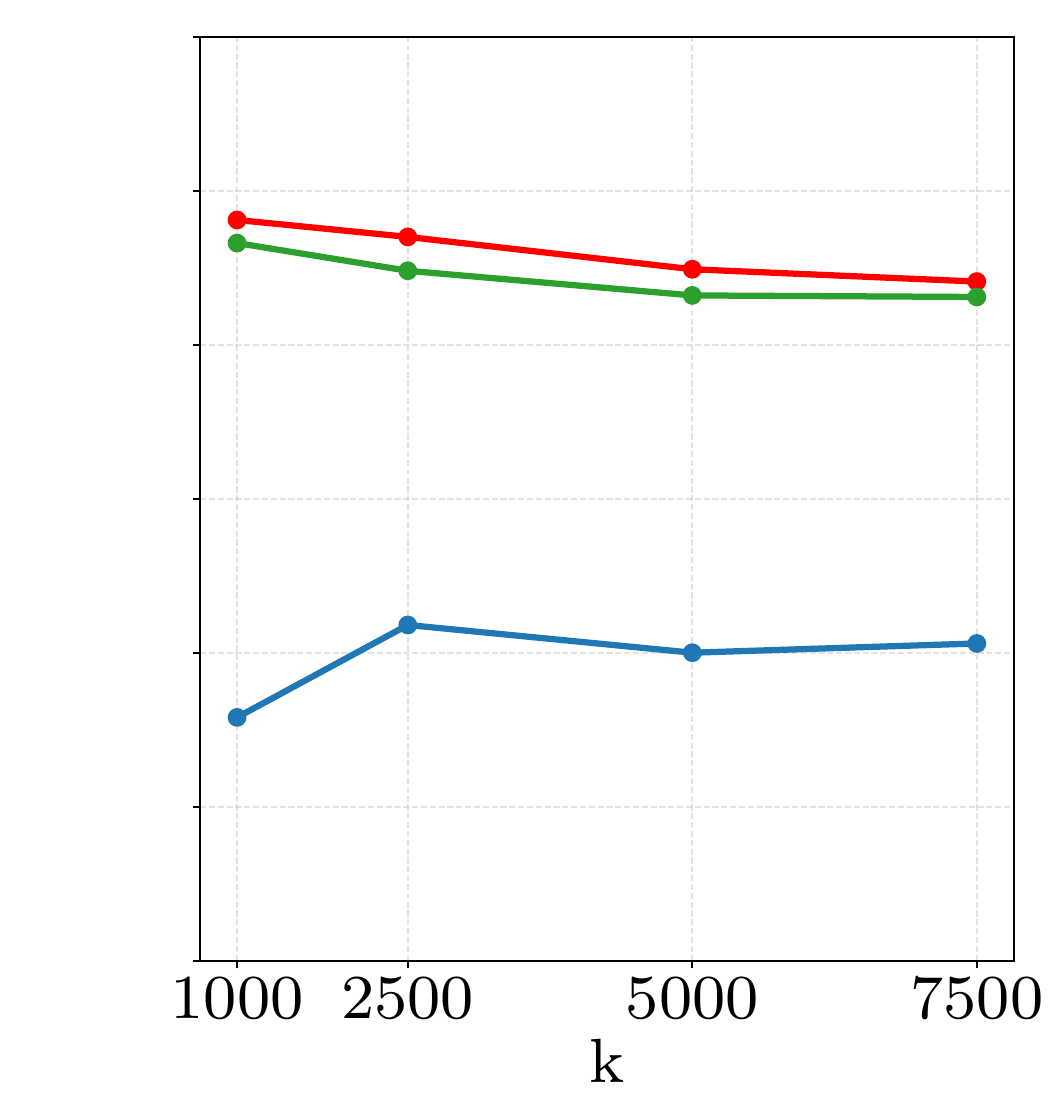}
        \caption{$n=5{,}000$}
    \end{subfigure}
    \hfill
    \begin{subfigure}[b]{0.24\linewidth}
        \centering
        \includegraphics[width=\linewidth]{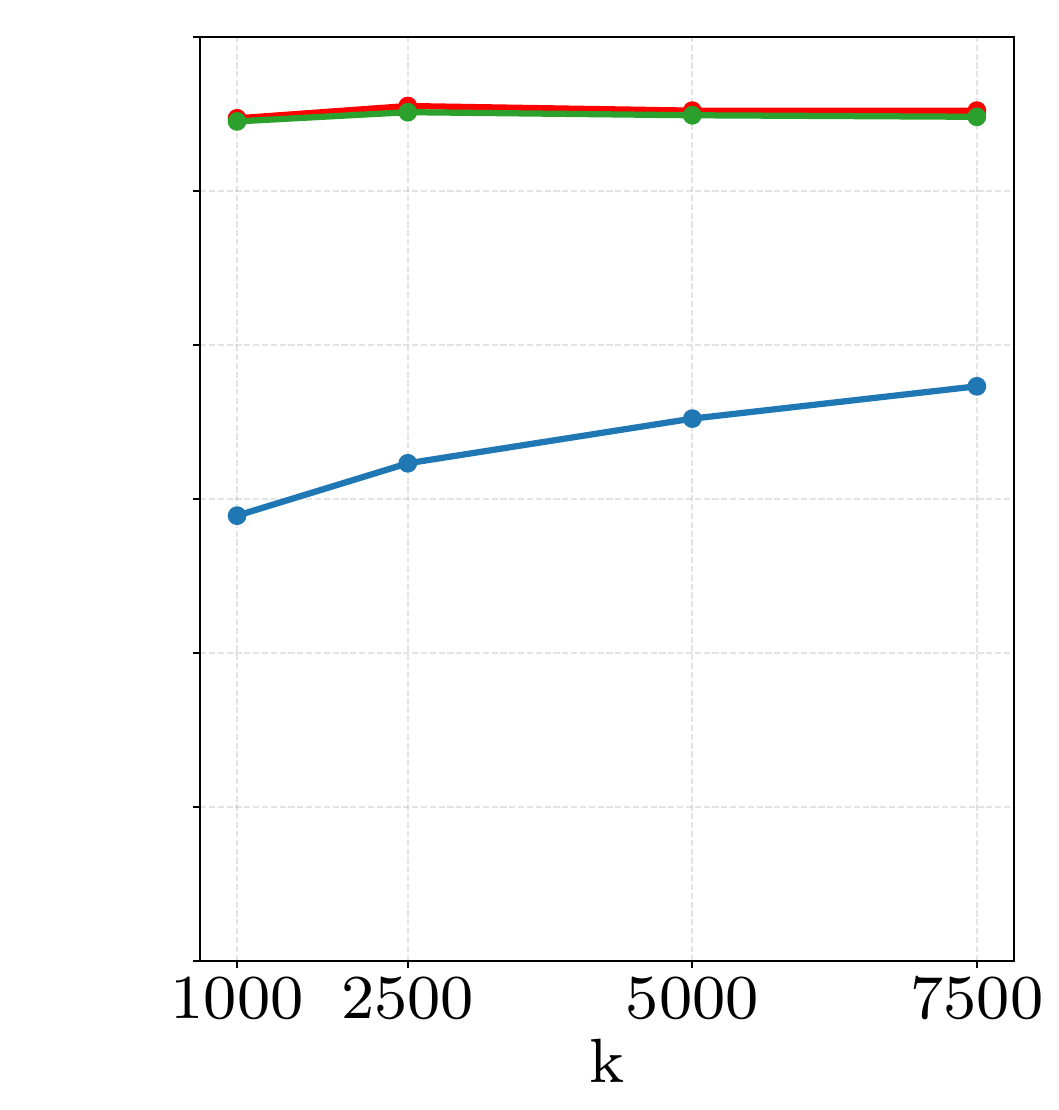}
        \caption{$n=10{,}000$}
    \end{subfigure}
    \hfill
    \begin{subfigure}[b]{0.24\linewidth}
        \centering
        \includegraphics[width=\linewidth]{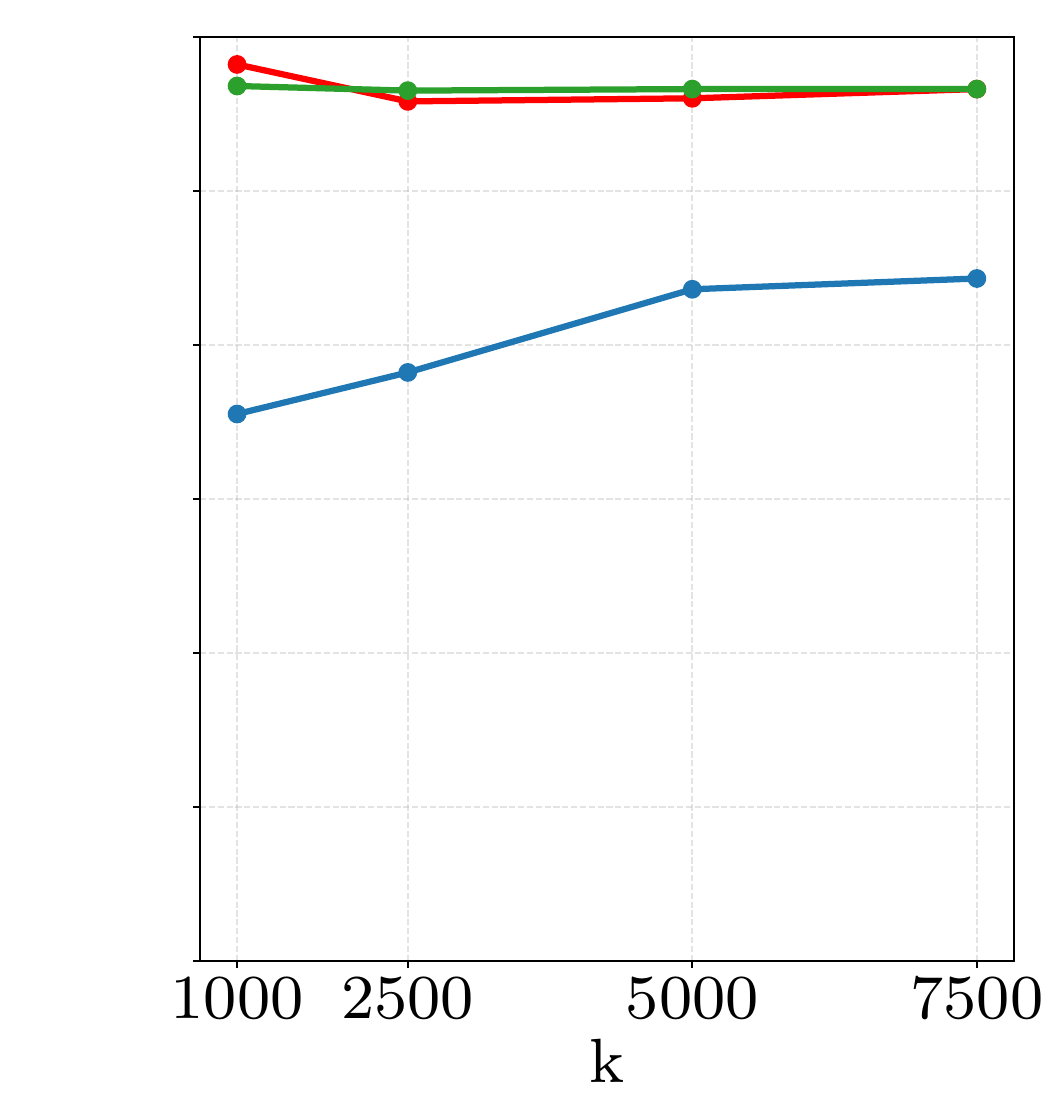}
        \caption{$n=15{,}000$}
    \end{subfigure}

    \vspace{0.1cm} 

    \centering
    \includegraphics[width=0.6\linewidth, keepaspectratio]{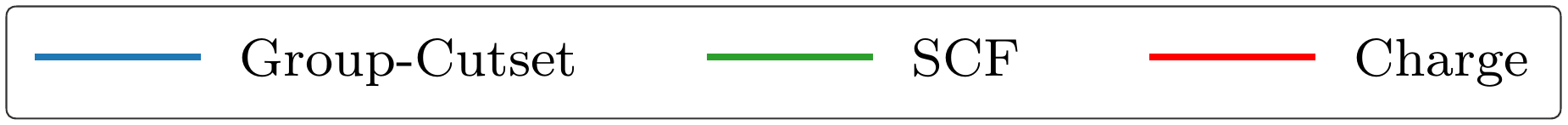}

    \caption{Optimality gap on large-scale instances after $500\,\mathrm{s}$ for varying graph sizes ($n$) and number of POIs ($k$).}
    \label{fig:large_scale_comp}
\end{figure}

Fig.~\ref{fig:large_scale_comp} presents the optimality gap as a function of the number of POIs for different roadmap sizes.
\algname{Group-Cutset} maintains substantially smaller optimality gaps as instance size grows, whereas \algname{Charge} exhibits limited gap reduction across all tested scales.
The \algname{SCF} formulation shows moderate improvements on smaller instances but degrades rapidly as problem size increases.
This behavior reflects fundamental differences in how global connectivity is enforced in the corresponding LP relaxations.
Both \algname{SCF} and \algname{Charge} rely on global flow or charge-balance constraints to eliminate subtours~\cite{gavish1978travelling,wong1980integer}, which allow fractional solutions to satisfy connectivity by spreading flow across many weakly selected edges.
As the graph size grows, such fractional connectivity patterns lead to increasingly weak lower bounds~\cite{miller1960integer,mizutani2024leveraging}.
In contrast, \algname{Group-Cutset}  enforces connectivity through explicit root-to-group cut constraints.
Any feasible solution, fractional or integral, must allocate sufficient total edge capacity across cuts separating the root from each POI group.
As the number of POIs increases, the growing number of such constraints progressively tightens the LP relaxation, explaining the more stable optimality gaps observed for \algname{Group-Cutset} on large instances.

\subsection{Separation Oracle Design Evaluation}
\label{sec:oracle-eval}
We evaluate the impact of separation-oracle design on lower-bound strength and solver convergence of the \algname{Group-Cutset} BnC solver (Sec.~\ref{sep-oracle-design}). 
We consider three variants:
(i) a \emph{connectivity-based oracle};
(ii) a \emph{flow-based oracle} applied exhaustively to all groups;
(iii) a \emph{combined oracle} that validates integral candidates via connectivity checks and applies flow-based separation to a uniformly sampled subset of groups at fractional solutions.

Fig.~\ref{fig:sampling-ablation}(a) compares the final optimality gaps obtained using these variants on real-world instances. 
Using only the connectivity-based oracle (crimson) yields large final gaps across all instances, indicating a weak LP relaxation in which fractional solutions remain poorly constrained, leading to weak lower bounds and slow convergence.
In contrast, the combined oracle (purple) substantially strengthens the formulation, consistently reducing the final optimality gap by approximately 10--30\% relative to connectivity-only separation. 
This improvement stems from the ability of the flow-based approach to generate effective constraints, even when validating only a small sample of groups.
We also evaluated a standalone flow-based oracle (turquoise bars in Fig.~\ref{fig:sampling-ablation}(a)) that performs exhaustive root-to-group separation. 
Although strong in principle, this approach is computationally prohibitive: validating a single candidate requires solving thousands of max-flow problems, leading to excessive separation time and negligible progress. 
As a result, exhaustive flow-based separation does not scale to large \gip instances.
Fig.~\ref{fig:sampling-ablation}(b) examines the effect of group-sampling size in the combined oracle. 
Final optimality gaps vary by only a few percentage points across a wide range of values, indicating low sensitivity to this parameter and enabling effective separation with modest samples.

\begin{figure}[t]
    \centering
    \begin{subfigure}{0.49\linewidth}
        \centering
        \includegraphics[width=\linewidth]{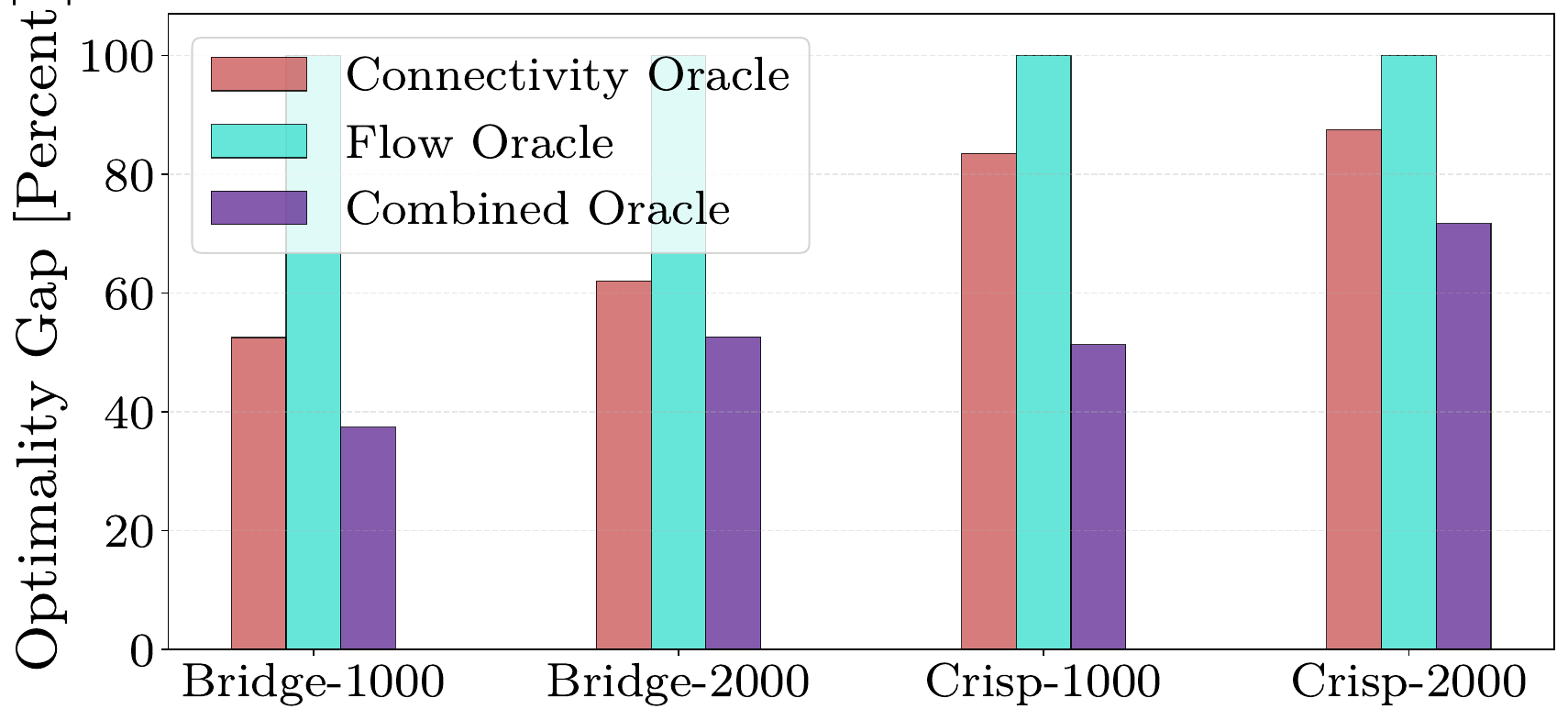}
        \caption{Comparison between separation oracles.}
        \label{fig:sampling-a}
    \end{subfigure}
    \hfill
    \begin{subfigure}{0.47\linewidth}
        \centering
        \includegraphics[width=\linewidth]{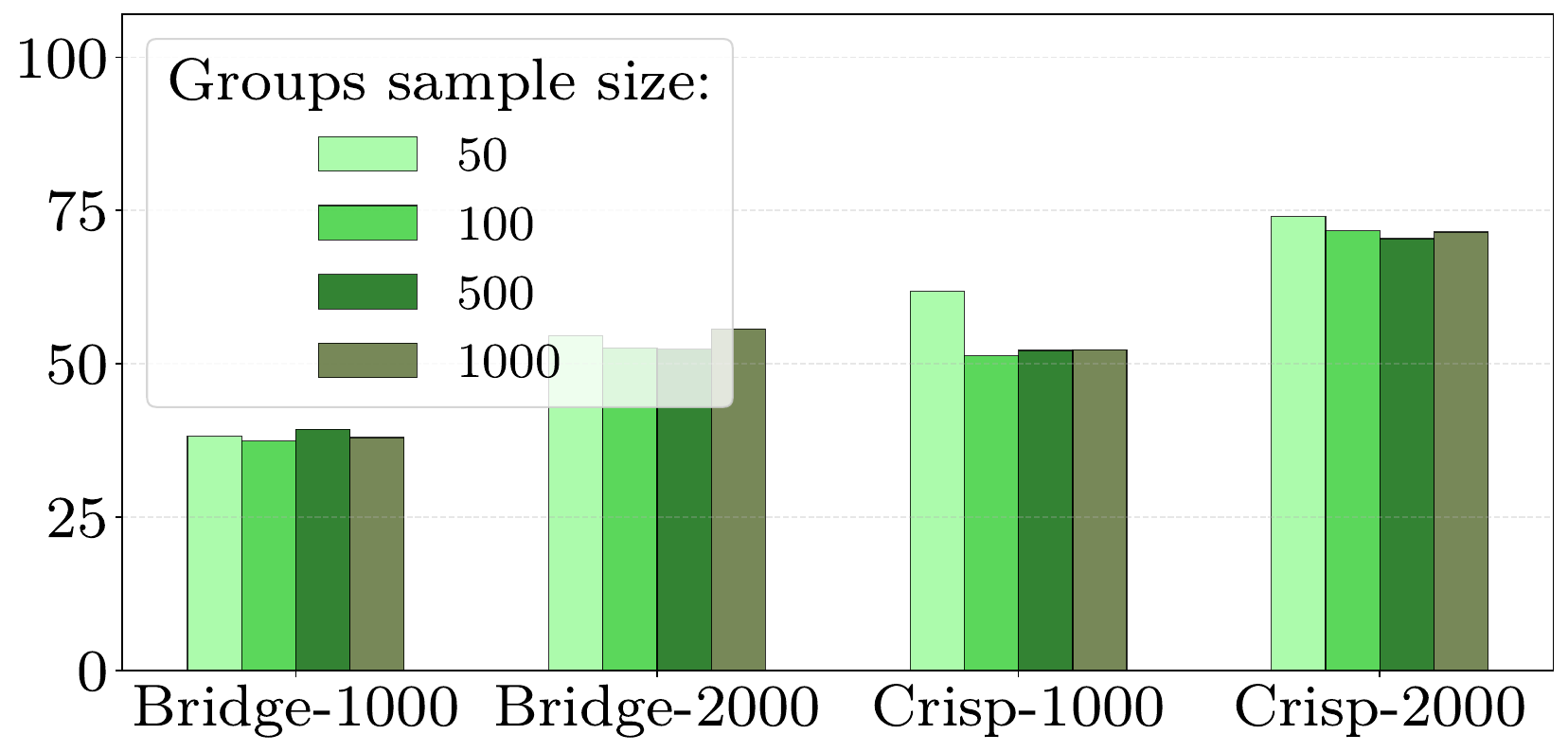}
        \caption{Effect of group sampling size.}
        \label{fig:sampling-b}
    \end{subfigure}
    \caption{Ablation study for the separation oracles reporting the final optimality gaps after
    $500\,\mathrm{s}$ on the real-world instances.
    (a)~Comparison of connectivity-only, flow-only, and combined separation
        oracles (using a group sample size of 100).
    (b)~Effect of group sampling size for the combined oracle.}
    \label{fig:sampling-ablation}
\end{figure}

Taken together, these results highlight the importance of separation-oracle design for scalability, with the combined oracle achieving a practical balance between strength and efficiency.

\section{Conclusion and Future Work}
In this work, we have positioned the graph inspection planning (\gip) problem within the broader context of graph-based optimization, highlighting its algorithmic connections to Steiner tree, TSP and network flow. Leveraging this perspective, we developed and analyzed three distinct MILP formulations, with our primary contribution being a scalable Branch-and-Cut solver centered on the \algname{Group-Cutset} formulation. 
Beyond inspection planning, our results suggest that lazy formulations within the Branch-and-Cut framework may provide a powerful paradigm for a broader class of graph-based planning problems, including those with capacity constraints, priority structures, or specialized objectives.
Several promising directions remain for future research.
As for the inspection planning problem, our results reveal significant variation in each solver performance across problem settings. Understanding the sources of this variation, and particularly how it relates to inspection structure and the locality of sensor–POI visibility, appears to be a promising direction for developing solvers tailored to specific, high-impact use cases.
From a combinatorial optimization perspective, several promising directions remain.
First, we plan to further explore the trade-off between motion-planning fidelity and optimization complexity, as increasingly dense roadmaps improve geometric accuracy but significantly complicate approximation quality and optimality certification.
Second, we aim to strengthen the solver for \gip by incorporating advanced MILP techniques, including additional classes of cutting planes and primal heuristics.
Finally, we see significant potential in extending the proposed framework to richer settings, such as multi-robot inspection planning, partial and adaptive inspection objectives, and online or incremental problem formulations.

\niceparagraph{Acknowledgments.} 
Large language models (ChatGPT and Gemini) were used for light editing and grammar refinement, as well as limited assistance with preliminary literature exploration and figure presentation and formatting.


\bibliographystyle{plainnat}
\bibliography{references}

@inproceedings{Panasoff.Solovey.25,
  author       = {Itai Panasoff and
                  Kiril Solovey},
  title        = {Effective Sampling for Robot Motion Planning Through the Lens of
 Lattices},
  booktitle      = {Robotics: Science and Systems},
  year         = {2025},
}

@article{karaman2010incremental,
  title={\href{https://arxiv.org/pdf/1005.0416
}{Incremental sampling-based algorithms for optimal motion planning}},
  author={Karaman, Sertac and Frazzoli, Emilio},
  journal={Robotics Science and Systems VI},
  volume={104},
  number={2},
  pages={267--274},
  year={2010}
}

@article{fu2019toward,
  title={\href{https://www.roboticsproceedings.org/rss15/p57.pdf}{Toward asymptotically-optimal inspection planning via efficient near-optimal graph search}},
  author={Fu, Mengyu and Kuntz, Alan and Salzman, Oren and Alterovitz, Ron},
  journal={Robotics science and systems: online proceedings},
  volume={2019},
  pages={10--15607},
  year={2019}
}

@inproceedings{fu2021computationally,
  title={\href{https://ieeexplore.ieee.org/document/9561653}{Computationally-efficient roadmap-based inspection planning via incremental lazy search}},
  author={Fu, Mengyu and Salzman, Oren and Alterovitz, Ron},
  booktitle={2021 IEEE International Conference on Robotics and Automation (ICRA)},
  pages={7449--7456},
  year={2021},
  organization={IEEE}
}

@article{mizutani2024leveraging,
  title={\href{https://arxiv.org/pdf/2407.00251}{Leveraging fixed-parameter tractability for robot inspection planning}},
  author={Mizutani, Yosuke and Salomao, Daniel Coimbra and Crane, Alex and Bentert, Matthias and Drange, P{\aa}l Gr{\o}n{\aa}s and Reidl, Felix and Kuntz, Alan and Sullivan, Blair D},
  journal={arXiv preprint arXiv:2407.00251},
  year={2024}
}

@misc{gurobi,
  author = {{Gurobi Optimization, LLC}},
  title = {\href{https://www.gurobi.com}{Gurobi Optimizer Reference Manual}},
  year = 2024,
}

@inproceedings{christofides2022worst,
  title={\href{https://link.springer.com/content/pdf/10.1007/s43069-021-00101-z.pdf}{Worst-case analysis of a new heuristic for the travelling salesman problem}},
  author={Christofides, Nicos},
  booktitle={Operations Research Forum},
  volume={3},
  number={1},
  pages={20},
  year={2022},
  organization={Springer}
}

@article{takahashi1980approximate,
  title={\href{https://cir.nii.ac.jp/crid/1570009750462176256}{An approximate solution for {Steiner} problem in graphs}},
  author={Takahashi, Hiromitsu},
  journal={Math. Japonica},
  volume={24},
  number={6},
  pages={573--577},
  year={1980}
}

@article{koch1998solving,
  title={\href{https://www.academia.edu/download/101109106/kochmartin1998.pdf}{Solving {Steiner} tree problems in graphs to optimality}},
  author={Koch, Thorsten and Martin, Alexander},
  journal={Networks: An International Journal},
  volume={32},
  number={3},
  pages={207--232},
  year={1998},
  publisher={Wiley Online Library}
}

@article{garg2000polylogarithmic,
  title={A polylogarithmic approximation algorithm for the group {Steiner} tree problem},
  author={Garg, Naveen and Konjevod, Goran and Ravi, Ramamoorthi},
  journal={Journal of Algorithms},
  volume={37},
  number={1},
  pages={66--84},
  year={2000},
  publisher={Elsevier}
}

@inproceedings{halperin2003polylogarithmic,
  title={Polylogarithmic inapproximability},
  author={Halperin, Eran and Krauthgamer, Robert},
  booktitle={Proceedings of the thirty-fifth annual ACM symposium on Theory of computing},
  pages={585--594},
  year={2003}
}

@article{chekuri2006greedy,
  title={A greedy approximation algorithm for the group {Steiner} problem},
  author={Chekuri, Chandra and Even, Guy and Kortsarz, Guy},
  journal={Discrete Applied Mathematics},
  volume={154},
  number={1},
  pages={15--34},
  year={2006},
  publisher={Elsevier}
}

@article{gamrath2017scip,
  title={SCIP-Jack—a solver for {STP} and variants with parallelization extensions},
  author={Gamrath, Gerald and Koch, Thorsten and Maher, Stephen J and Rehfeldt, Daniel and Shinano, Yuji},
  journal={Mathematical Programming Computation},
  volume={9},
  number={2},
  pages={231--296},
  year={2017},
  publisher={Springer}
}

@book{cormen2009introduction,
  title     = {Introduction to Algorithms},
  author    = {Cormen, Thomas H. and Leiserson, Charles E. and Rivest, Ronald L. and Stein, Clifford},
  edition   = {3},
  year      = {2009},
  publisher = {The MIT Press},
  address   = {Cambridge, MA, USA},
}

@article{dantzig2003max,
  title={On the max flow min cut theorem of networks},
  author={Dantzig, George and Fulkerson, Delbert Ray},
  journal={Linear inequalities and related systems},
  volume={38},
  pages={225--231},
  year={2003}
}

@article{bertacco2007feasibility,
  title={A feasibility pump heuristic for general mixed-integer problems},
  author={Bertacco, Livio and Fischetti, Matteo and Lodi, Andrea},
  journal={Discrete Optimization},
  volume={4},
  number={1},
  pages={63--76},
  year={2007},
  publisher={Elsevier}
}

@article{miller1960integer,
  title={Integer programming formulation of traveling salesman problems},
  author={Miller, Clair E and Tucker, Albert W and Zemlin, Richard A},
  journal={Journal of the ACM (JACM)},
  volume={7},
  number={4},
  pages={326--329},
  year={1960},
  publisher={ACM New York, NY, USA}
}

@techreport{gavish1978travelling,
  title={The travelling salesman problem and related problems},
  author={Gavish, Bezalel and Graves, Stephen C},
  year={1978},
  number = {OR-078-78},
  institution={Massachusetts Institute of Technology, Operations Research Center}
}

@article{fischetti2005feasibility,
  title={The feasibility pump},
  author={Fischetti, Matteo and Glover, Fred and Lodi, Andrea},
  journal={Mathematical Programming},
  volume={104},
  number={1},
  pages={91--104},
  year={2005},
  publisher={Springer}
}

@article{danna2005exploring,
  title={Exploring relaxation induced neighborhoods to improve {MIP} solutions},
  author={Danna, Emilie and Rothberg, Edward and Pape, Claude Le},
  journal={Mathematical Programming},
  volume={102},
  number={1},
  pages={71--90},
  year={2005},
  publisher={Springer}
}

@incollection{applegate2001tsp,
  title={{TSP} cuts which do not conform to the template paradigm},
  author={Applegate, David and Bixby, Robert and Chv{\'a}tal, Va{\v{s}}ek and Cook, William},
  booktitle={Computational Combinatorial Optimization: Optimal or Provably Near-Optimal Solutions},
  pages={261--303},
  year={2001},
  publisher={Springer}
}

@article{atkar2005uniform,
  title={Uniform coverage of automotive surface patches},
  author={Atkar, Prasad N and Greenfield, Aaron and Conner, David C and Choset, Howie and Rizzi, Alfred A},
  journal={The International Journal of Robotics Research},
  volume={24},
  number={11},
  pages={883--898},
  year={2005},
  publisher={SAGE Publications}
}

@inproceedings{cho2021planning,
  title={Planning sensing sequences for subsurface {3D} tumor mapping},
  author={Cho, Brian Y and Hermans, Tucker and Kuntz, Alan},
  booktitle={2021 international symposium on medical robotics (ISMR)},
  pages={1--7},
  year={2021},
  organization={IEEE}
}

@inproceedings{cho2024efficient,
  title={Efficient and Accurate Mapping of Subsurface Anatomy via Online Trajectory Optimization for Robot Assisted Surgery},
  author={Cho, Brian Y and Kuntz, Alan},
  booktitle={2024 IEEE International Conference on Robotics and Automation (ICRA)},
  pages={15478--15484},
  year={2024},
  organization={IEEE}
}

@inproceedings{cheng2008time,
  title={Time-optimal {UAV} trajectory planning for {3D} urban structure coverage},
  author={Cheng, Peng and Keller, James and Kumar, Vijay},
  booktitle={2008 IEEE/RSJ International Conference on Intelligent Robots and Systems},
  pages={2750--2757},
  year={2008},
  organization={IEEE}
}

@article{karp1972reducibility,
  author  = {Richard M. Karp},
  title   = {Reducibility Among Combinatorial Problems},
  journal = {Complexity of Computer Computations},
  year    = {1972},
  pages   = {85--103}
}

@book{lawler1985tsp,
  author    = {Eugene L. Lawler and Jan Karel Lenstra and Alexander H. G. Rinnooy Kan and David B. Shmoys},
  title     = {The Traveling Salesman Problem},
  publisher = {Wiley},
  year      = {1985}
}

@article{feige1998threshold,
  author  = {Uriel Feige},
  title   = {A Threshold of $\ln n$ for Approximating Set Cover},
  journal = {Journal of the ACM},
  volume  = {45},
  number  = {4},
  pages   = {634--652},
  year    = {1998}
}

@article{sahni1976p,
  author  = {Sartaj Sahni and Teofilo Gonzalez},
  title   = {P-complete Approximation Problems},
  journal = {Journal of the ACM},
  volume  = {23},
  number  = {3},
  pages   = {555--565},
  year    = {1976}
}

@book{ahuja1994network,
  title={Network flows: theory, algorithms and applications},
  author={Ahuja, Ravindra K and Magnanti, Thomas L and Orlin, James B},
  year={1994},
  publisher={Prentice hall}
}

@article{edmonds1972theoretical,
  title={Theoretical improvements in algorithmic efficiency for network flow problems},
  author={Edmonds, Jack and Karp, Richard M},
  journal={Journal of the ACM (JACM)},
  volume={19},
  number={2},
  pages={248--264},
  year={1972},
  publisher={ACM New York, NY, USA}
}

@article{claus1984new,
  title={A new formulation for the travelling salesman problem},
  author={Claus, A},
  journal={SIAM Journal on Algebraic Discrete Methods},
  volume={5},
  number={1},
  pages={21--25},
  year={1984},
  publisher={SIAM}
}

@article{anderson2017continuum,
  title={Continuum reconfigurable parallel robots for surgery: Shape sensing and state estimation with uncertainty},
  author={Anderson, Patrick L and Mahoney, Arthur W and Webster, Robert James},
  journal={IEEE robotics and automation letters},
  volume={2},
  number={3},
  pages={1617--1624},
  year={2017},
  publisher={IEEE}
}

@inproceedings{mahoney2016reconfigurable,
  title={Reconfigurable parallel continuum robots for incisionless surgery},
  author={Mahoney, Arthur W and Anderson, Patrick L and Swaney, Philip J and Maldonado, Fabien and Webster, Robert J},
  booktitle={2016 IEEE/RSJ International Conference on Intelligent Robots and Systems (IROS)},
  pages={4330--4336},
  year={2016},
  organization={IEEE}
}

@book{hwang1992steiner,
  title={The Steiner Tree Problem},
  author={Hwang, Frank K. and Richards, Dana S. and Winter, Pawel},
  publisher={Elsevier},
  year={1992}
}

@article{laporte1983generalized,
  title={Generalized travelling salesman problem through n sets of nodes: an integer programming approach},
  author={Laporte, Gilbert and Nobert, Yves},
  journal={INFOR: Information Systems and Operational Research},
  volume={21},
  number={1},
  pages={61--75},
  year={1983},
  publisher={Taylor \& Francis}
}

@incollection{applegate2011traveling,
  title={The traveling salesman problem: a computational study},
  author={Applegate, David L and Bixby, Robert E and Chv{\'a}tal, Va{\v{s}}ek and Cook, William J},
  booktitle={The Traveling Salesman Problem},
  year={2011},
  publisher={Princeton university press}
}

@article{edmonds1965paths,
  title={Paths, trees, and flowers},
  author={Edmonds, Jack},
  journal={Canadian Journal of Mathematics},
  volume={17},
  pages={449--467},
  year={1965}
}

@article{cook1999computing,
  title={Computing minimum-weight perfect matchings},
  author={Cook, William J. and Rohe, Andre},
  journal={INFORMS Journal on Computing},
  volume={11},
  number={2},
  pages={138--148},
  year={1999}
}

@inproceedings{karp1981maximum,
  title={Maximum matching in sparse random graphs},
  author={Karp, Richard M. and Sipser, Michael},
  booktitle={Proceedings of the 22nd Annual Symposium on Foundations of Computer Science},
  pages={364--375},
  year={1981}
}

@book{vazirani2001approximation,
  title={Approximation Algorithms},
  author={Vazirani, Vijay V.},
  publisher={Springer},
  year={2001}
}

@article{lawler1966branch,
  title={Branch-and-bound methods: A survey},
  author={Lawler, Eugene L and Wood, David E},
  journal={Operations research},
  volume={14},
  number={4},
  pages={699--719},
  year={1966},
  publisher={INFORMS}
}

@incollection{land2009automatic,
  title={An automatic method for solving discrete programming problems},
  author={Land, Ailsa H and Doig, Alison G},
  booktitle={50 Years of Integer Programming 1958-2008: From the Early Years to the State-of-the-Art},
  pages={105--132},
  year={2009},
  publisher={Springer}
}

@article{padberg1991branch,
  title={A branch-and-cut algorithm for the resolution of large-scale symmetric traveling salesman problems},
  author={Padberg, Manfred and Rinaldi, Giovanni},
  journal={SIAM review},
  volume={33},
  number={1},
  pages={60--100},
  year={1991},
  publisher={SIAM}
}

@article{mitchell2002branch,
  title={Branch-and-cut algorithms for combinatorial optimization problems},
  author={Mitchell, John E},
  journal={Handbook of applied optimization},
  volume={1},
  number={1},
  pages={65--77},
  year={2002},
  publisher={Oxford, UK}
}

@book{bondy1979graph,
  title={Graph theory with applications},
  author={Bondy, John Adrian and Murty, Uppaluri Siva Ramachandra},
  year={1979},
  publisher={north-Holland}
}

@inproceedings{wong1980integer,
  title={Integer programming formulations of the traveling salesman problem},
  author={Wong, Richard T},
  booktitle={Proceedings of the IEEE international conference of circuits and computers},
  volume={149},
  pages={152},
  year={1980},
  organization={IEEE Press Piscataway NJ}
}

\newpage
\appendix

\section{Approximation Bounds for Graph Inspection Planning} \label{sec:Approx-bounds}
In this section, we leverage the connection between \gip and \gst to establish
approximation upper and lower bounds for \gip. For upper bound, \citet{mizutani2024leveraging} Thm.~4 stated a straight forward $O(k)$-approximation algorithm.
To the best of our knowledge, approximation lower bounds have not previously been established for \gip.\footnote{The discussion below considers the \emph{undirected} version of \gip.}
We begin by restating the following approximation bounds for \gst. Below, $k$ denotes the number of groups, and $n=|V|$. 
\begin{theorem}[Lower approximation bounds for \gst~\cite{halperin2003polylogarithmic}]\label{GSTlb}
    \gst is  hard\footnote{\label{fn:complex-assum}
    \citet{halperin2003polylogarithmic} rely on the assumption that NP has no quasi-polynomial Las Vegas algorithms. This is slightly stronger assumption than $\classP \neq \classNP$ but widely accepted as plausible.} to approximate by a factor of $\Omega(\log^{2-\epsilon}k)$, for any constant $\epsilon>0$. For Theorem \ref{GIPapproxLower} we assume the same computational model.
\end{theorem}
\begin{theorem}[Upper approximation bounds for \gst~\cite{chekuri2006greedy,garg2000polylogarithmic}]\label{GSTub}
    There exists a randomized polynomial time $O(\log n \cdot \log^2 k)$-approximation algorithm for \gst.
\end{theorem}

We recall that \gst was formally defined in Def.~\ref{def:gst}.
Although the problem is stated in terms of finding a tree, an equivalent formulation may allow any connected subgraph, since an optimal solution can always be taken to be a tree.

To derive analogous approximation bounds for \gip, we rely on the following two
observations:
\begin{itemize}
    \item[\textbf{O1}] Any feasible solution to \gip induces a valid solution to \gst, since a \gip tour is, by definition, a connected subgraph that covers all groups.
    \item[\textbf{O2}] Any feasible solution to \gst can be transformed into a feasible solution to \gip by traversing the solution tree in a depth-first manner and returning along each edge.
    This transformation yields a \gip tour whose cost is at most twice the cost of the original \gst solution.
\end{itemize}

The following theorem provides a lower approximation bound for \gip.
\begin{theorem}[Lower approximation bound for \gip]
\label{GIPapproxLower}
    \gip is hard$^{\ref{fn:complex-assum}}$ to approximate within a factor of $\Omega(\log^{2-\epsilon}k)$ for all constant $\epsilon>0$.
\end{theorem}

\begin{proof}
    Assume by contradiction we are given an $\alpha$-approximation oracle $\mathcal{O}$ for \gip, such that $\alpha<\Omega(\log^{2-\epsilon}k)$. 
    Given a \gst instance, we consider the corresponding \gip and apply $\mathcal{O}$ to obtain a tour $\tau_{\mathcal{O}}$ that $\alpha$-approximates $\tau^*$, the optimal \gip tour. 

    Let~$T^*$ be an optimal \gst solution. 
    By traversing $T^*$ back and forth and following Observation \textbf{O2}, we can construct a closed tour for \gip, $T^*_p$, such that $|T^*_p|\leq 2 \cdot |T^*|$.
    As $T^*_p$ is a covering-tour for \gip and $\tau^*$ is the optimal tour we have that
    \[|\tau^*|\leq |T^*_p|.\]
    Using the assumption that $\mathcal{O}$ is an $\alpha$-approximation algorithm, we have that
    \[|\tau_\mathcal{O}|\leq \alpha \cdot |\tau^*| \leq 2\alpha\cdot |T^*|.\]
    This solution $\tau_\mathcal{O}$ forms a covering and connected edge subset, and hence is a solution to the \gst problem by Observation \textbf{O1} (and may be tightened to a tree by the removal of cycle-closing edge). 
    Hence we found a solution $\tau_\mathcal{O}$ to the \gst problem that is $2\cdot\alpha<\Omega(\log^{2-\epsilon}k)$ approximating the optimal \gst solution, in contradiction to Thm~\ref{GSTlb}. 
    Hence no such approximation algorithm exists for \gip under the computational model of \citet{halperin2003polylogarithmic} \qed
\end{proof}

The following theorem, describing upper approximation bounds for \gip, improves  the previous $O(k)$ approximation result provided by~\citet{mizutani2024leveraging}.
\begin{theorem}[Upper approximation bounds for \gip] \label{GIPapproxUpper}
    There exists an efficient $O(\log n \cdot \log^2k)$-approximation algorithm for \gip.
\end{theorem}

\begin{proof}
    Given a \gip instance with optimal solution $T^*$,
    we apply the approximation algorithm $A_{\gst}$ provided by Thm.~\ref{GSTub}, yielding a solution $T_A$ to the \gst problem for which $|T_A|\leq O(\log n \cdot \log^2k)\cdot |T^*|$. 
    We build a tour $\tau$ such that $|\tau|\leq 2\cdot|T_A|$ traveling $T_A$ back and forth by Observation \textbf{O2}.
    However, as the optimal \gip solution $\tau^*$ is also a \gst solution by Observation \textbf{O1}, $|T^*|\leq |\tau^*|$, it follows that 
    \[|\tau|\leq 2\cdot|T_A| \leq 2\cdot O(\log n \cdot \log^2k)\cdot|T^*|\leq O(\log n \cdot \log^2k) \cdot |\tau^*|.\]
    Thus,   $A_{gst}$ is an efficient $O(\log n \log^2 k)$-approximation algorithm for \gip. \qed
\end{proof}

\section{Partial-coverage \gip} \label{sec:partial-cover}
In this section we briefly describe how to extend the baseline MILP formulation for \gip
\eqref{eq:gtsp-milp-obj}--\eqref{eq:gtsp-binary} to support \emph{partial} group coverage. 
This extension is orthogonal to the choice of subtour-elimination constraints, as it can be plugged into any of the discussed formulations. 
It is based on ideas proposed by \citet{mizutani2024leveraging}, and we include it here for completeness.

Given a desired number of POIs to be covered, $q \le k$, we introduce additional binary decision variables $z_i \in \{0,1\}$ for all $i \in K$, indicating whether POI $i$ is selected for inspection. The requirement to cover at least $q$ POIs is enforced by the constraint
\begin{equation}
\sum_{i \in K} z_i \ge q .
\end{equation}

The group-coverage constraint~\eqref{eq:gtsp-groups} is then modified so that coverage is enforced only for selected POIs:
\begin{equation}
\sum_{v \in S_i} \sum_{u \in N^-(v)} x_{uv} \ge z_i,
\quad \forall i \in K .
\label{eq:gtsp-groups-choice}
\end{equation}

This \gip variant is particularly interesting in the aspect of model tightness tradeoffs. 
On one hand, partial coverage relaxes the full \gip problem constraints, easing the discovery of feasible solutions. 
On the other hand, the additional combinatorial explosion of having to chose $q$ out of the $k$ POIs significantly increases the solution space. 
When viewed through the lens of MILP, it looses model integrality gap, as fractional 
solutions may spread partial coverage across many groups via the $z_i$ variables, weakening coverage constraints in the LP relaxation. Eventually, this leads to a problem where finding the \emph{optimal} partial-coverage solution may be even more challenging than finding the optimal full-coverage solution. We leave empirical evaluation of this variant to future work.

\section{Additional Details on the Primal Heuristic for \gip} \label{appendix:gip-heuristic}
We provide additional details regarding our \gip-specific primal heuristic. The heuristic consists of the following three phases (visualized in Fig.~\ref{fig:heuristic_phases} and~\ref{fig:PH_phase3} and detailed shortly). 
At a high level, we first modify edge costs based on the current LP relaxation, enabling the heuristic to leverage guidance from fractional solutions encountered during the search. Then, inspired by the \gst problem~\cite{koch1998solving, takahashi1980approximate}, we greedily construct a tree that connects the root to at least one vertex from each POI group. Finally, we transform this group-covering tree into a feasible \gip tour by adding a small set of connecting edges. 

We mention that a common approach to generate incumbent solutions in MILP problems is to round fractional solutions to enforce integrality. The LP-based cost modification strategy~\cite{koch1998solving} generalizes this idea by using fractional values to guide the construction of integer solutions indirectly: instead of rounding variables explicitly, edge costs are discounted according to their fractional values in the current relaxation. 
This biases the heuristic toward structures already supported by the LP relaxation, improving solution quality and alignment with the solver’s ongoing search. From this perspective, our heuristic can be viewed as a problem-aware LP-rounding step embedded within the BnC process.

We now provide a detailed description for each of the three phases. 

\niceparagraph{Phase~1-LP-based cost-discounted graph.}
Let $\{x_e\}_{e\in E}$ denote a fractional solution obtained from an LP relaxation solved during the BnC process.
We construct a new weighted graph $G_n = (V, E, c_n)$ over the same vertices and edges as the original graph $G = (V, E, c)$, where each edge $e \in E$ is assigned a modified cost
$
c_n(e) := c(e)\cdot(1 - x_e).
$
This modification reflects the intuition that edges partially selected by the LP relaxation are more likely to belong to high-quality integer solutions. By lowering their cost, the heuristic is encouraged to follow the structure suggested by the relaxation, while still allowing alternative edges when necessary.

\niceparagraph{Phase~2-Group-covering tree.}
In the second phase, illustrated in Fig.~\ref{fig:heuristic_phases}, we construct a group-covering tree $T$ rooted at $r$ in the discounted graph~$G_n$ that covers all groups, i.e.,~$\forall S_i \in \mathcal{S},~T \cap S_i \neq \emptyset$.
To this end, we maintain a set~$\mathcal{U}$ of uncovered groups, initialized as $\mathcal{U} := \mathcal{S}$, and incrementally expand the tree, which initially consists of the root vertex $r$ alone.

Before growing the tree, we compute all-pairs shortest paths in $G_n$, obtaining for every pair of vertices $u,v \in V$ a shortest (weighted) path $\pi^*(u,v)$ and its cost~$d^*(u,v)$ with respect to the modified edge costs $c_n$. This forms the computational bottleneck of the heuristic, as it is recalculated whenever a new LP-relaxed solution is found and graph weights are updated.
At each iteration, we define the set of candidate vertices
\[
V_{\text{covered}} := \{ v \in V \setminus T \mid \exists S_i \in \mathcal{U} \text{ such that } v \in S_i \},
\]
consisting of vertices that belong to at least one uncovered group.
Among these candidates, we greedily select the vertex $v' \in V_{\text{covered}}$ that is closest to the current tree, i.e.,
\[
v' := \arg\min_{v \in V_{\text{covered}}} \min_{w \in T} d^*(w,v).
\]
The selected vertex $v'$ is added to the tree by inserting the shortest path $\pi^*(u,v')$, where
\[
u := \arg\min_{w \in T} d^*(w,v').
\]
This process is repeated until $\mathcal{U}=\emptyset$, yielding a tree that intersects all groups in~$\mathcal{S}$.

\usetikzlibrary{shapes.geometric, arrows, positioning}

\definecolor{logicblue}{HTML}{729fcf}
\definecolor{edgegreen}{HTML}{6ab06a}
\definecolor{dashedorange}{HTML}{f9d9a5}
\definecolor{nodeblue}{HTML}{d9e8fb}
\definecolor{nodeorange}{HTML}{fef0d1}
\definecolor{nodered}{HTML}{f8cecc}
\definecolor{nodegray}{HTML}{e1e1e1}

\tikzset{
    common canvas/.style={scale=0.45, transform shape, baseline=(current bounding box.center)},
    main node/.style={circle, draw, ultra thick, minimum size=9mm, font=\large\sffamily},
    edge label/.style={font=\small, inner sep=3pt, fill=white},
    base/.style={circle, draw, ultra thick, minimum size=7mm, inner sep=0pt},
    standard/.style={base, fill=white},
    bigLabelNode/.style={base, fill=white, font=\Large\bfseries},
    ghost/.style={base, dashed, draw=gray!70},
    blueNode/.style={ghost, fill=nodeblue},
    orangeNode/.style={ghost, fill=nodeorange},
    redNode/.style={ghost, fill=nodered},
    grayNode/.style={ghost, fill=nodegray},
    solidEdge/.style={draw=edgegreen, line width=1.8pt}, 
    dashedEdge/.style={draw=black!60, line width=1.2pt, dashed},
    orangeEdge/.style={draw=dashedorange, line width=1.8pt, dashed}
}

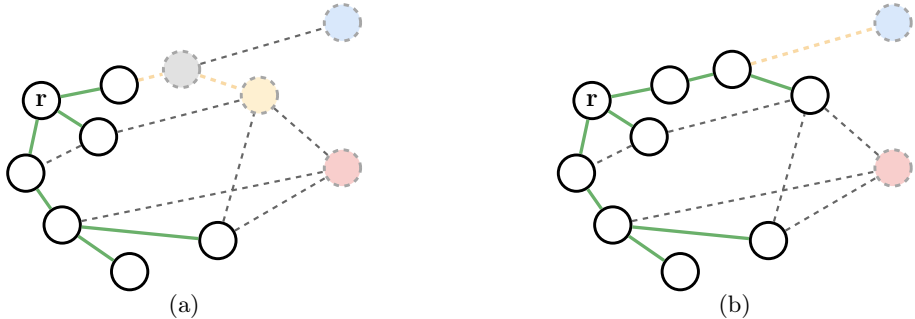
\begin{figure}[t]
    \centering
    \begin{minipage}{0.4\linewidth}
        \centering
        \resizebox{\linewidth}{!}{
            \begin{tikzpicture}[baseline=(current bounding box.center)]
                \node[bigLabelNode] (r) at (0,3) {r}; 
                \node[standard] (n1) at (1.5,3.3) {};
                \node[standard] (n2) at (1.1,2.3) {};
                \node[standard] (n3) at (-0.3,1.6) {};
                \node[standard] (n4) at (0.4,0.6) {};
                \node[standard] (n5) at (3.4,0.3) {};
                \node[standard] (n6) at (1.7,-0.3) {};
                \node[grayNode] (g_gray) at (2.7,3.6) {};
                \node[blueNode] (g_blue) at (5.8,4.5) {};
                \node[orangeNode] (g_orange) at (4.2,3.1) {};
                \node[redNode] (g_red) at (5.8,1.7) {};
                \path[solidEdge] (r)--(n1) (r)--(n2) (r)--(n3) (n3)--(n4) (n4)--(n5) (n4)--(n6);
                \path[dashedEdge] (g_gray)--(g_blue) (n2)--(g_orange) (g_orange)--(g_red) (n4)--(g_red) (n5)--(g_orange) (n5)--(g_red) (n2)--(n3);
                \path[orangeEdge] (n1)--(g_gray) (g_gray)--(g_orange);
            \end{tikzpicture}
        }
        \centerline{\small (a)}
    \end{minipage}\hfill
    \begin{minipage}{0.4\linewidth}
        \centering
        \resizebox{\linewidth}{!}{
            \begin{tikzpicture}[baseline=(current bounding box.center)]
                \node[bigLabelNode] (r) at (0,3) {r}; 
                \node[standard] (n1) at (1.5,3.3) {};
                \node[standard] (n2) at (1.1,2.3) {};
                \node[standard] (n3) at (-0.3,1.6) {};
                \node[standard] (n4) at (0.4,0.6) {};
                \node[standard] (n5) at (3.4,0.3) {};
                \node[standard] (n6) at (1.7,-0.3) {};
                \node[standard] (n7) at (2.7,3.6) {};
                \node[standard] (n8) at (4.2,3.1) {};
                \node[blueNode] (g_blue) at (5.8,4.5) {};
                \node[redNode] (g_red) at (5.8,1.7) {};
                \path[solidEdge] (r)--(n1) (r)--(n2) (r)--(n3) (n3)--(n4) (n4)--(n5) (n1)--(n7) (n7)--(n8) (n4)--(n6);
                \path[dashedEdge] (n2)--(n8) (n8)--(g_red) (n4)--(g_red) (n5)--(n8) (n5)--(g_red) (n2)--(n3);
                \path[orangeEdge] (n7)--(g_blue);
            \end{tikzpicture}
        }
        \centerline{\small (b)}
    \end{minipage}
    \caption{Illustration of Phase 2 of the primal heuristic for \gip: A group covering tree (green) is built from the root $r$ towards the closest vertex among all  the vertices covering an uncovered group (colored vertices); (a) The yellow vertex is chosen as the closest vertex to the tree and a path from nearest tree vertex (orange edges) is found. (b) After the path was added, the tree continues to grow greedily until all groups are covered.}\label{fig:heuristic_phases}
\end{figure}

\niceparagraph{Phase~3-Tree-to-tour conversion.}
In the final phase (illustrated in Fig.~\ref{fig:PH_phase3}), we transform the group-covering tree $T$ into a feasible \gip tour by adding a minimum-cost set of edges so as to make $T$ an Eulerian subgraph.

Concretely, let $O$ denote the set of vertices in $T$ with odd degree. A standard approach to making the graph Eulerian is to compute a minimum-weight perfect matching $M$ on the complete graph induced by $O$, where edge weights correspond to shortest-path distances in $G_n$, and to add the matched shortest paths to~$T$. This procedure follows the matching phase of Christofides' algorithm~\cite{christofides2022worst} and can be implemented using Blossom-based algorithms for minimum-weight perfect matching~\cite{cook1999computing,edmonds1965paths}. The matching adds exactly one incident edge to each odd-degree vertex, resulting in the augmented graph $T+M$ being Eulerian. In our context, using an optimal matching minimizes the additional cost required to make $T$ traversable and typically yields higher-quality tours.


We refer to the heuristic described so far as \emph{covering tree + minimum matching}. 
For scalability and speed, we consider a lightweight alternative based on \emph{greedy} nearest-neighbor matching on $O$, in which the two closest unmatched vertices are repeatedly paired and removed~\cite{karp1981maximum,vazirani2001approximation}. We refer to this heuristic as \emph{covering tree + greedy matching}. 

In practice, the greedy matching variant substantially reduces heuristic computation time and often produces solutions close to those obtained with optimal matching at a fraction of the computational cost. The trade-off between solution quality and runtime for these two variants is evaluated in Sec.~\ref{heuristic-eval}: while minimum-weight perfect matching gives slightly improved heuristics, the greedy matching approach enables better scaling to very large \gip instances. After augmenting $T$ using the chosen matching strategy, we extract an Eulerian tour and shortcut repeated vertices to obtain the final \gip solution.

Throughout the experiments conducted in Sec.~\ref{sec:evaluation}, the \emph{covering tree + greedy matching} was the applied heuristic.

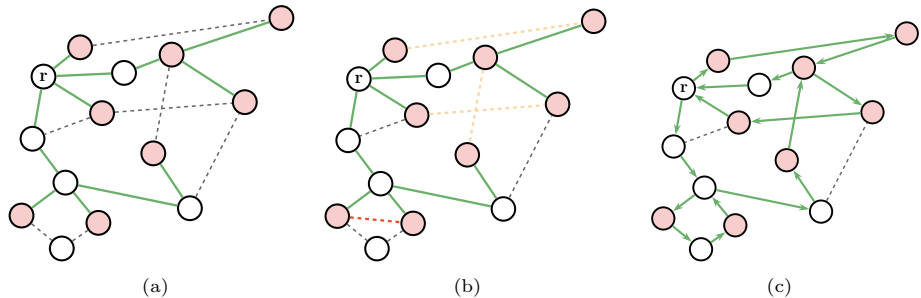
\begin{figure}[tbp]
     \centering
     \begin{subfigure}[b]{0.32\textwidth}
         \centering
         
        \resizebox{\textwidth}{!}{\begin{tikzpicture}[
    base/.style={circle, draw, ultra thick, minimum size=6.5mm, inner sep=0pt},
    standard/.style={base, fill=white},
    visited/.style={base, fill=nodered},
    root/.style={base, fill=white, minimum size=6.5mm, font=\Large\bfseries}, 
    treeEdge/.style={draw=edgegreen, line width=1.8pt},
    queryEdge/.style={draw=black!60, line width=1.2pt, dashed}
]

    \node[root] (r) at (0,3) {r}; 
    
    \node[visited] (v1) at (1,3.8) {};
    \node[standard] (n1) at (2.2,3.1) {};
    \node[visited] (v2) at (3.5,3.6) {};
    \node[visited] (v3) at (6.5,4.6) {};
    \node[visited] (v4) at (1.6,2) {};
    \node[visited] (v5) at (5.5,2.3) {};
    \node[standard] (n2) at (-0.3,1.3) {};
    \node[standard] (n3) at (0.6,0.1) {};
    \node[visited] (v7) at (-0.6,-0.8) {};
    \node[visited] (v8) at (1.5,-1) {};
    \node[standard] (n4) at (0.5,-1.7) {};
    \node[standard] (n5) at (4, -0.6) {};
    \node[visited] (v9) at (3,0.9) {};

    \draw[treeEdge] (r) -- (v1);
    \draw[treeEdge] (r) -- (n1);
    \draw[treeEdge] (n1) -- (v2) -- (v3);
    \draw[treeEdge] (r) -- (v4);
    \draw[treeEdge] (r) -- (n2) -- (n3);
    \draw[treeEdge] (n3) -- (v7);
    \draw[treeEdge] (n3) -- (v8);
    \draw[treeEdge] (n3) -- (n5);
    \draw[treeEdge] (n5) -- (v9);
    \draw[treeEdge] (v2) -- (v5);

    \draw[queryEdge] (v1) -- (v3);
    \draw[queryEdge] (v4) -- (v5);
    \draw[queryEdge] (v4) -- (n2);
    \draw[queryEdge] (v2) -- (v9);
    \draw[queryEdge] (n5) -- (v5);
    \draw[queryEdge] (v7) -- (n4) -- (v8);

\end{tikzpicture}
        }
         \caption{}
         \label{fig:panel_a}
     \end{subfigure}
     \hfill
     \begin{subfigure}[b]{0.32\textwidth}
         \centering
         \resizebox{\textwidth}{!}{\definecolor{dashedred}{HTML}{EE4B2B}
\definecolor{dashedblue}{HTML}{0096FF}

\begin{tikzpicture}[
    base/.style={circle, draw, ultra thick, minimum size=6.5mm, inner sep=0pt},
    standard/.style={base, fill=white},
    visited/.style={base, fill=nodered},
    root/.style={base, fill=white, minimum size=6.5mm, font=\Large\bfseries},
    treeEdge/.style={draw=edgegreen, line width=1.8pt},
    queryEdge/.style={draw=black!60, line width=1.2pt, dashed},
    blueEdge/.style={draw=dashedblue, line width=1.8pt, dashed},
    redEdge/.style={draw=dashedred, line width=1.8pt, dashed}
]

    \node[root] (r) at (0,3) {r}; 
    \node[visited] (v1) at (1,3.8) {};
    \node[standard] (n1) at (2.2,3.1) {};
    \node[visited] (v2) at (3.5,3.6) {};
    \node[visited] (v3) at (6.5,4.6) {};
    \node[visited] (v4) at (1.6,2) {};
    \node[visited] (v5) at (5.5,2.3) {};
    \node[standard] (n2) at (-0.3,1.3) {};
    \node[standard] (n3) at (0.6,0.1) {};
    \node[visited] (v7) at (-0.6,-0.8) {};
    \node[visited] (v8) at (1.5,-1) {};
    \node[standard] (n4) at (0.5,-1.7) {};
    \node[standard] (n5) at (4, -0.6) {};
    \node[visited] (v9) at (3,0.9) {};
    
    \draw[treeEdge] (r) -- (v1);
    \draw[treeEdge] (r) -- (n1);
    \draw[treeEdge] (n1) -- (v2) -- (v3);
    \draw[treeEdge] (r) -- (v4);
    \draw[treeEdge] (r) -- (n2) -- (n3);
    \draw[treeEdge] (n3) -- (v7);
    \draw[treeEdge] (n3) -- (v8);
    \draw[treeEdge] (n3) -- (n5);
    \draw[treeEdge] (n5) -- (v9);
    \draw[treeEdge] (v2) -- (v5);

    \draw[blueEdge] (v1) -- (v3);
    \draw[blueEdge] (v4) -- (v5);
    \draw[queryEdge] (v4) -- (n2);
    \draw[blueEdge] (v2) -- (v9);
    \draw[queryEdge] (n5) -- (v5);
    \draw[queryEdge] (v7) -- (n4) -- (v8);

    \draw[redEdge] (v7) -- (v8);
    
\end{tikzpicture}
         }
         \caption{}
         \label{fig:panel_b}
     \end{subfigure}
     \hfill
     \begin{subfigure}[b]{0.32\textwidth}
         \centering
         \resizebox{\textwidth}{!}{\usetikzlibrary{arrows.meta}

\definecolor{edgegreen}{HTML}{6ab06a}
\definecolor{nodered}{HTML}{f8cecc}
\definecolor{dashedorange}{HTML}{f9d9a5}
\definecolor{dashedred}{HTML}{EE4B2B}

\begin{tikzpicture}[
    base/.style={circle, draw, ultra thick, minimum size=6.5mm, inner sep=0pt},
    standard/.style={base, fill=white},
    visited/.style={base, fill=nodered},
    root/.style={base, fill=white, minimum size=6.5mm, font=\Large\bfseries},
    OtreeEdge/.style={draw=edgegreen, line width=1.8pt, -{Stealth[length=3mm, width=2mm]}},
    queryEdge/.style={draw=black!60, line width=1.2pt, dashed},
    orangeEdge/.style={draw=dashedorange, line width=1.8pt, dashed},
    redEdge/.style={draw=dashedred, line width=1.8pt, dashed}
]

    \node[root] (r) at (0,3) {r}; 
    \node[visited] (v1) at (1,3.8) {};
    \node[standard] (n1) at (2.2,3.1) {};
    \node[visited] (v2) at (3.5,3.6) {};
    \node[visited] (v3) at (6.5,4.6) {};
    \node[visited] (v4) at (1.6,2) {};
    \node[visited] (v5) at (5.5,2.3) {};
    \node[standard] (n2) at (-0.3,1.3) {};
    \node[standard] (n3) at (0.6,0.1) {};
    \node[visited] (v7) at (-0.6,-0.8) {};
    \node[visited] (v8) at (1.5,-1) {};
    \node[standard] (n4) at (0.5,-1.7) {};
    \node[standard] (n5) at (4, -0.6) {};
    \node[visited] (v9) at (3,0.9) {};

    \draw[OtreeEdge] (r) -> (v1);
    \draw[OtreeEdge] (n1) -> (r);
    \draw[OtreeEdge] (v2) -> (n1);
    \draw[OtreeEdge] (v3) -> (v2);
    \draw[OtreeEdge] (v4) -> (r);
    \draw[OtreeEdge] (r) -> (n2);
    \draw[OtreeEdge] (n2)-> (n3);
    \draw[OtreeEdge] (n3) -- (v7);
    \draw[OtreeEdge] (v8) -> (n3);
    \draw[OtreeEdge] (n3) -> (n5);
    \draw[OtreeEdge] (n5) -- (v9);
    \draw[OtreeEdge] (v2) -- (v5);
    \draw[OtreeEdge] (v1) -- (v3);
    \draw[OtreeEdge] (v5) -> (v4);
    \draw[OtreeEdge] (v9) -> (v2);
    \draw[OtreeEdge] (v7) -> (n4); 
    \draw[OtreeEdge] (n4) -> (v8); 

    \draw[queryEdge] (v4) -- (n2);
    \draw[queryEdge] (n5) -- (v5);

\end{tikzpicture}
         }
         \caption{}
         \label{fig:panel_c}
     \end{subfigure}
        \caption{Illustration of Phase 3 of the primal heuristic for \gip: (a) The phase starts with a tree  (green) connecting the root to a covering set of vertices (red). (b) A matching is found between odd-degree tree vertices, resulting in an Eulerian graph. The matching may include actual graph edges (blue) or virtual edges representing shortest paths (red) (c) An Eulerian tour is imposed on the graph, orienting the selected edges to form a valid \gip tour.
        }
        \label{fig:PH_phase3}
\end{figure}





\section{Additional Experimental Results} \label{appendix:complement-eval}

In this section we complement the experimental discussion in Sec.~\ref{sec:evaluation} with the following: (i) We elaborate on the inspection-planning simulation, used in the \texttt{Controlled} instances. (ii) We complement the scale-range analysis with small scale experiments. (iii) We extend our ablation study regarding BnC algorithmic building blocks, with a primal-heuristic discussion.

\subsection{Inspection-Planning Simulator}\label{appendix:simulator}
While real-world datasets capture practical inspection scenarios, simulation enables independent control over key problem parameters such as graph size and number of POIs,
allowing us to isolate and study the scaling behavior of different MILP formulations.
To this end, we developed 
a configurable inspection-planning simulator that generates problem instances on a point robot model, preserving the structural characteristics of roadmap-based inspection
graphs, while enabling precise control over instance scale and coverage complexity.

The point robot operates in a three-dimensional configuration space
$\mathcal{C} = \mathbb{R}^2 \times \mathrm{SO}(2)$, representing planar position and orientation. It is equipped with a camera characterized by a fixed field-of-view (FOV) angle and a maximal inspection range. Each scenario is generated by randomly placing
two-dimensional obstacles, a set of POIs $\mathcal{P}$, and a designated starting configuration $r$ in a bounded planar workspace.

An RRG roadmap is then constructed in the robot’s configuration space,
respecting obstacle constraints using standard collision-checking procedures
\cite{karaman2010incremental}. Vertices correspond to collision-free robot configurations, and edges represent dynamically feasible local motions with associated traversal costs.
Based on the camera’s FOV and inspection range, each vertex $v \in V$ is assigned a visibility set $\chi(v) \subseteq \mathcal{P}$, consisting of the POIs that can be inspected when the robot is positioned at $v$. This process yields a graph inspection planning instance as defined in Def.~\ref{def:gipp}. In the experiments below, we systematically vary the roadmap size $n=|V|$ and the number of POIs $k=|\mathcal{P}|$ to analyze solver behavior across a wide range of problem scales.

The simulation environment is a planar maze populated with randomly placed L-shaped obstacles, inducing narrow passages and nontrivial connectivity patterns. POIs are uniformly distributed in the free space, and inspection plans must navigate the roadmap to ensure coverage while maintaining global connectivity. 
The process of generating a representative simulated instance is demonstrated in Fig.~\ref{fig:sim-building}. 
\begin{figure}[t]
    \centering
    \begin{subfigure}[t]{0.32\linewidth}
        \centering
        \includegraphics[height=4cm, width=\linewidth]{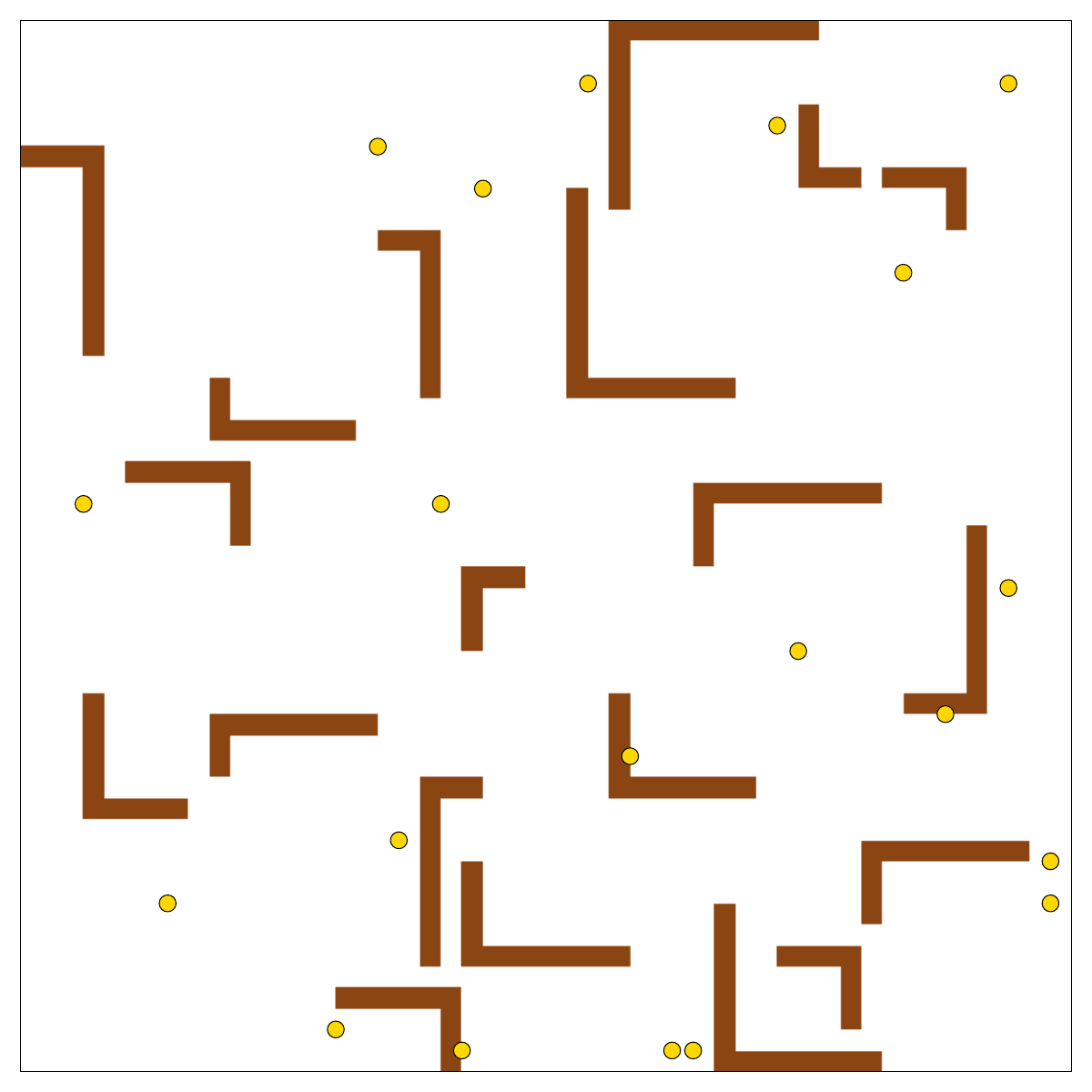}
        \caption{}
        \label{fig:sampling-a}
    \end{subfigure}\hfill
    \begin{subfigure}[t]{0.32\linewidth}
        \centering
        \includegraphics[height=4cm, width=\linewidth]{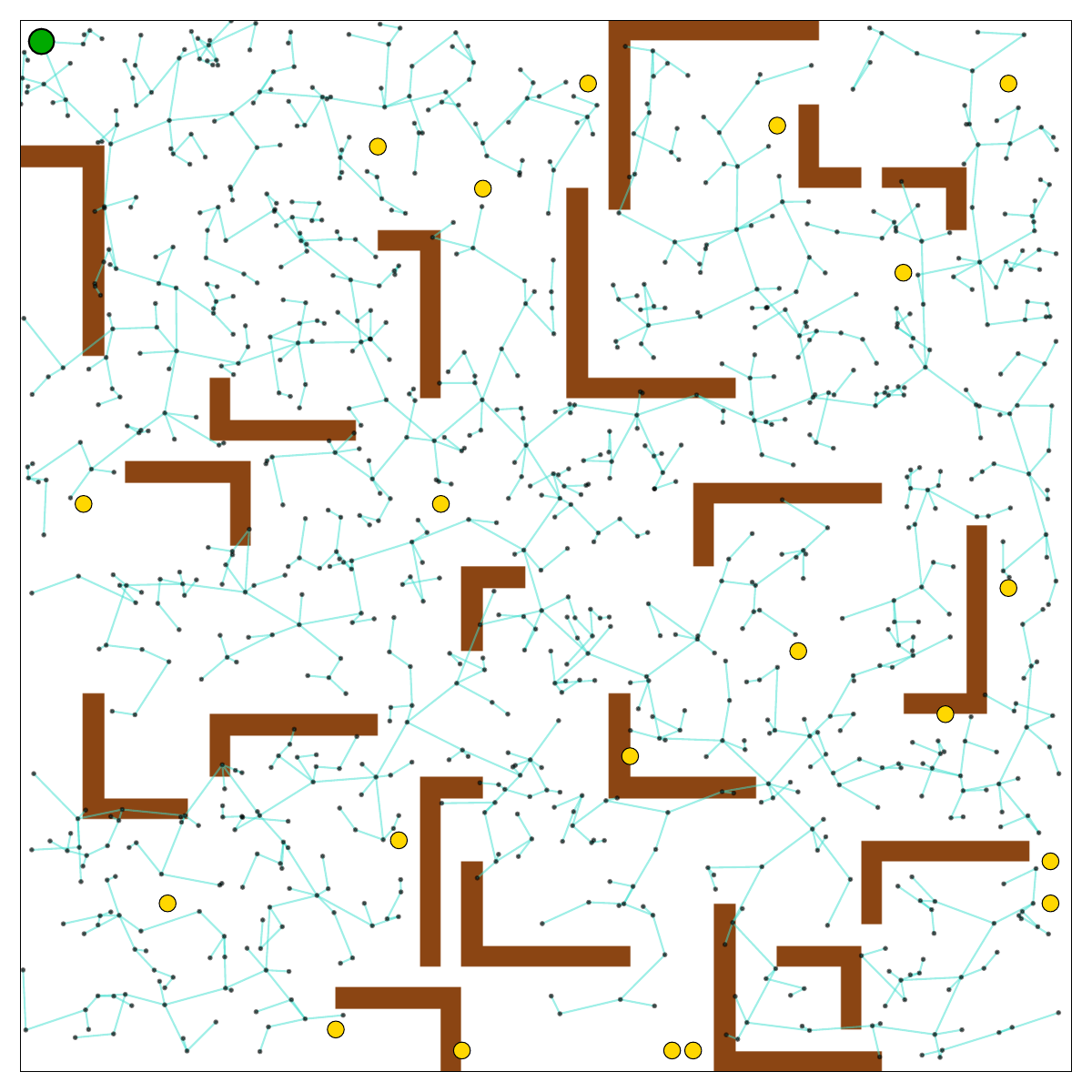}
        \caption{}
        \label{fig:sampling-b}
    \end{subfigure}\hfill
    \begin{subfigure}[t]{0.32\linewidth}
        \centering
        \includegraphics[height=4cm, width=\linewidth]{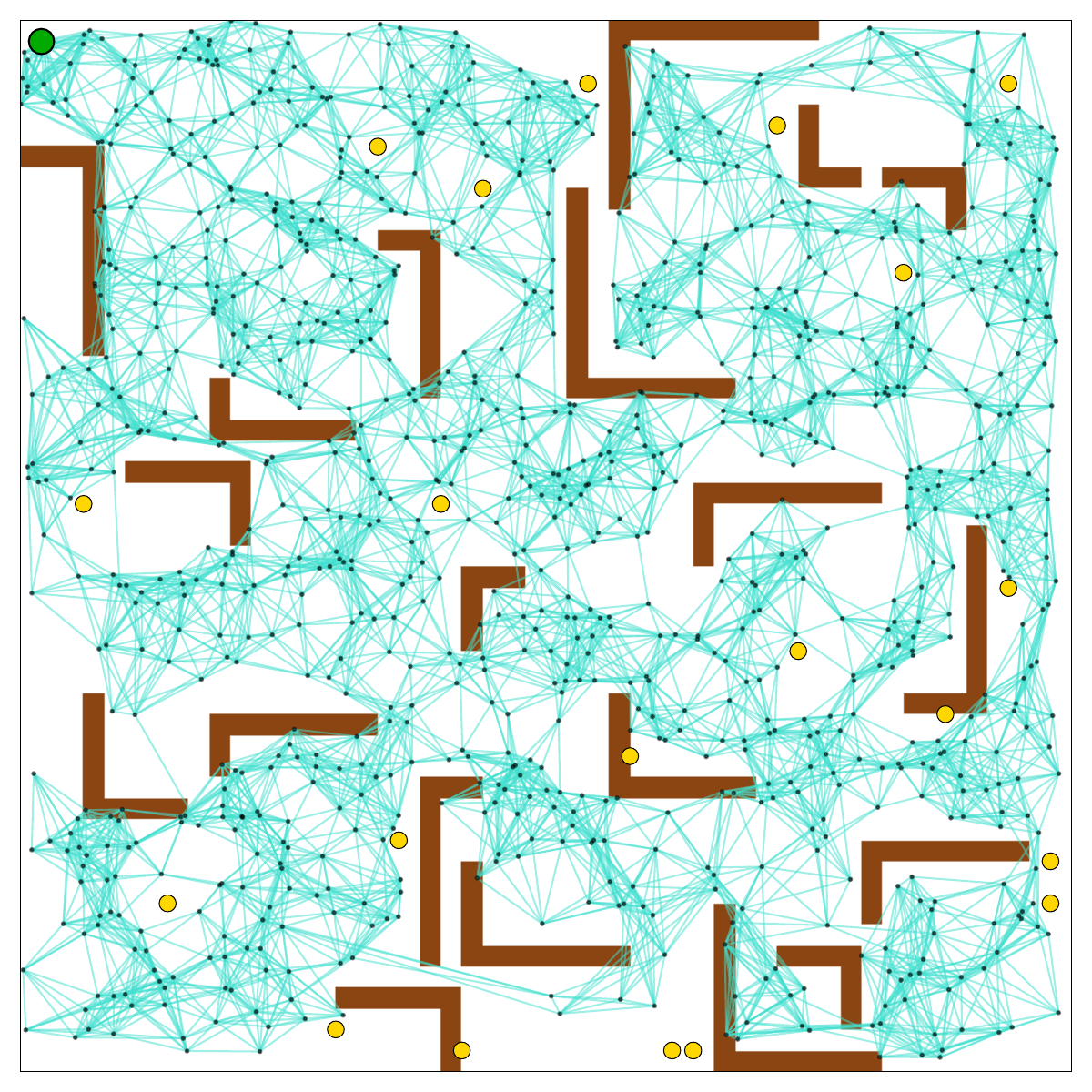}
        \caption{}
        \label{fig:sampling-c}
    \end{subfigure}
    \caption{Inspection planning simulation. (a) The simulation environment is a planar maze of given size, with randomly allocated L-shaped obstacles (in brown). $k=50$ POIs are scattered in the free space at random (in yellow). (b) RRT graph with $n=1,000$ vertices built to span the robot's configuration space with the robot's starting configuration located in the top-left corner and marked in green. 
    (c) Edges are locally added to create an RRG graph which is used as the inspection-planning graph.}
    \label{fig:sim-building}
\end{figure}

\subsection{Evaluation of Small Instances} 
\label{appendix:small-exp}
Recall that in Sec.~\ref{sec:form-comp} we studied the optimality gap on large-scale simulated instances (Fig.~\ref{fig:large_scale_comp}). We now examine the opposite end of the scale spectrum, where tighter formulations become feasible and exact convergence is sometimes attainable in the controlled scenarios. 
Fig.~\ref{fig:small_scale_comp} reports the final optimality gap after a time limit of $500\,\mathrm{s}$. Solver were applied in identical configuration to Sec.~\ref{sec:evaluation}.

\begin{figure}
    \centering
    \begin{subfigure}[b]{0.24\linewidth}
        \centering
        \includegraphics[width=\linewidth]{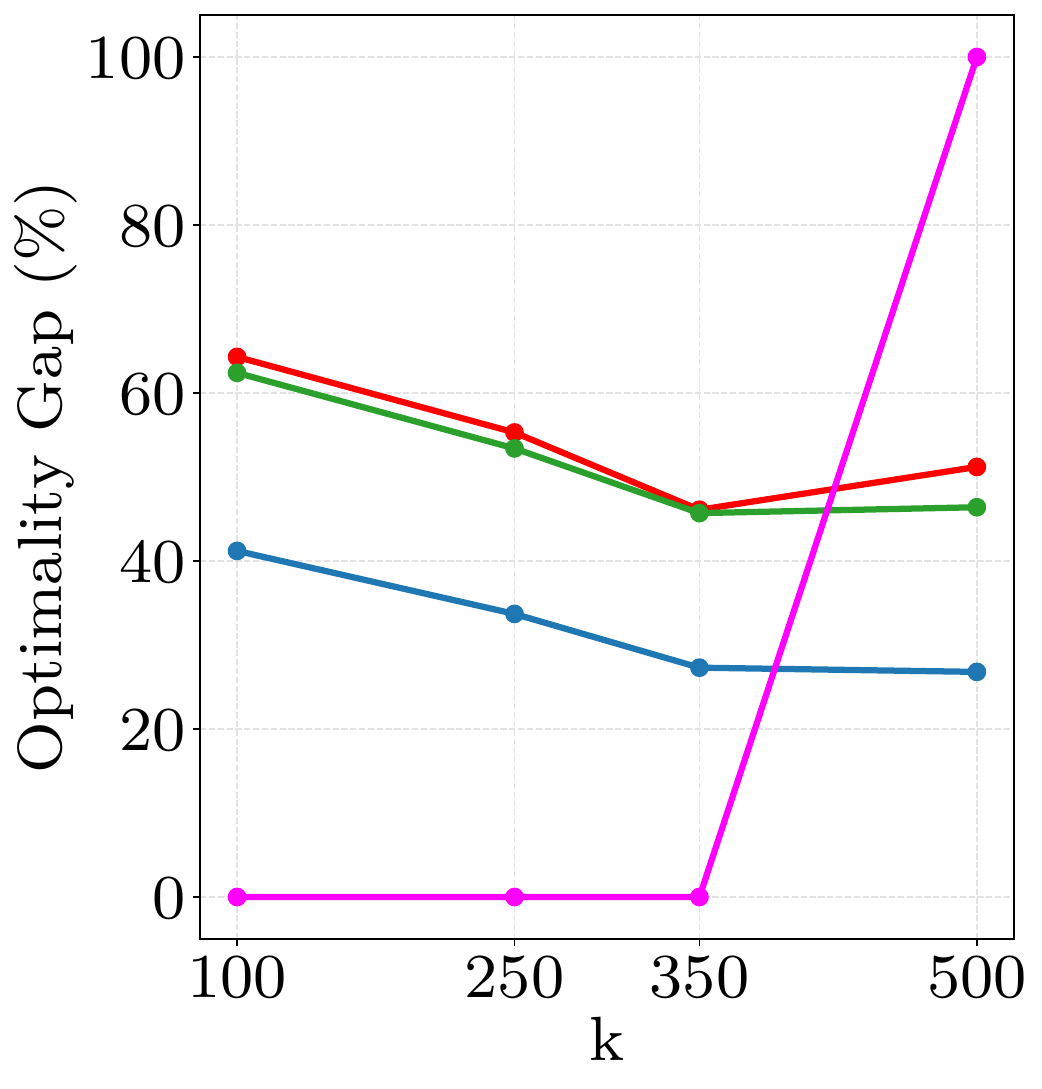}
        \caption{$n=250$}
    \end{subfigure}
    \hfill
    \begin{subfigure}[b]{0.24\linewidth}
        \centering
        \includegraphics[width=\linewidth]{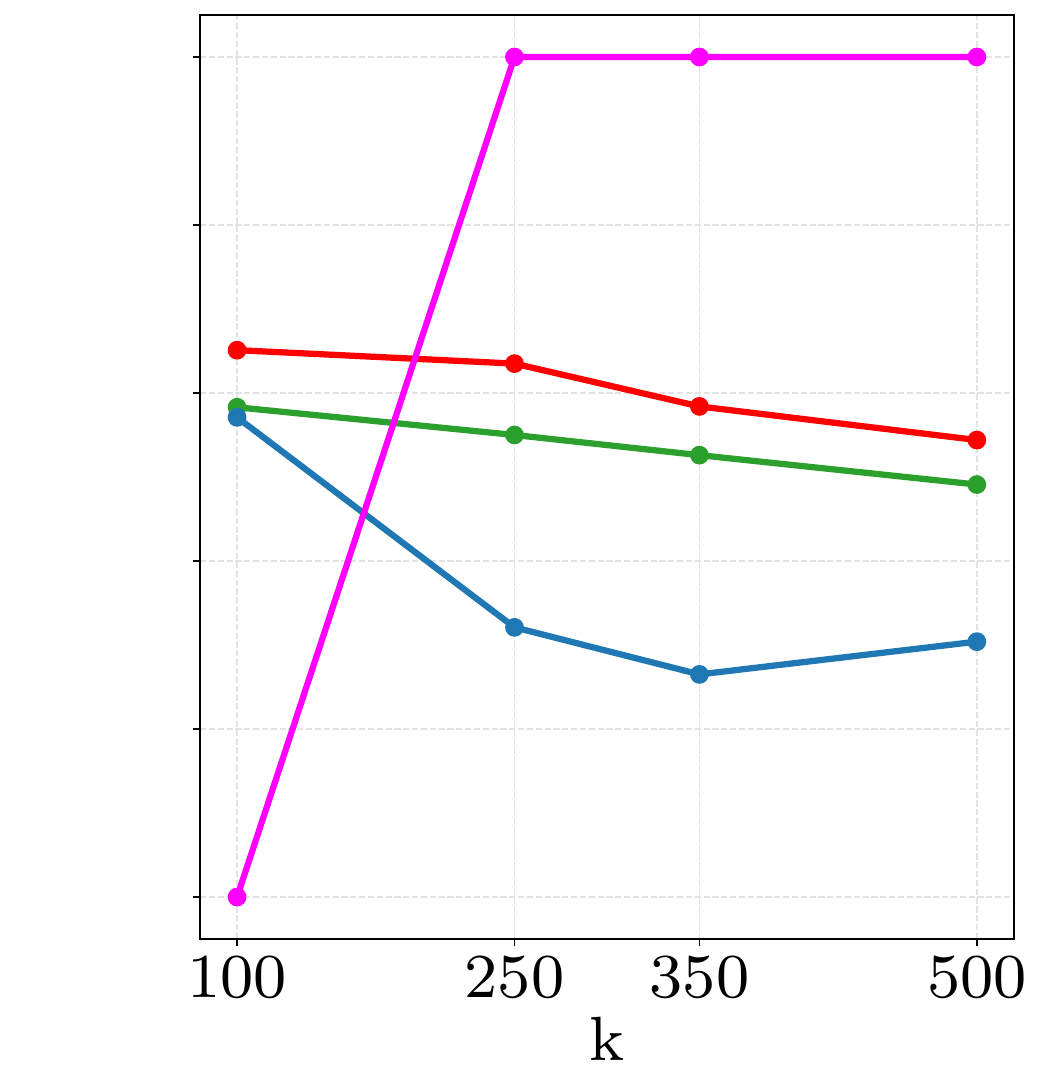}
        \caption{$n=500$}
    \end{subfigure}
    \hfill
    \begin{subfigure}[b]{0.24\linewidth}
        \centering
        \includegraphics[width=\linewidth]{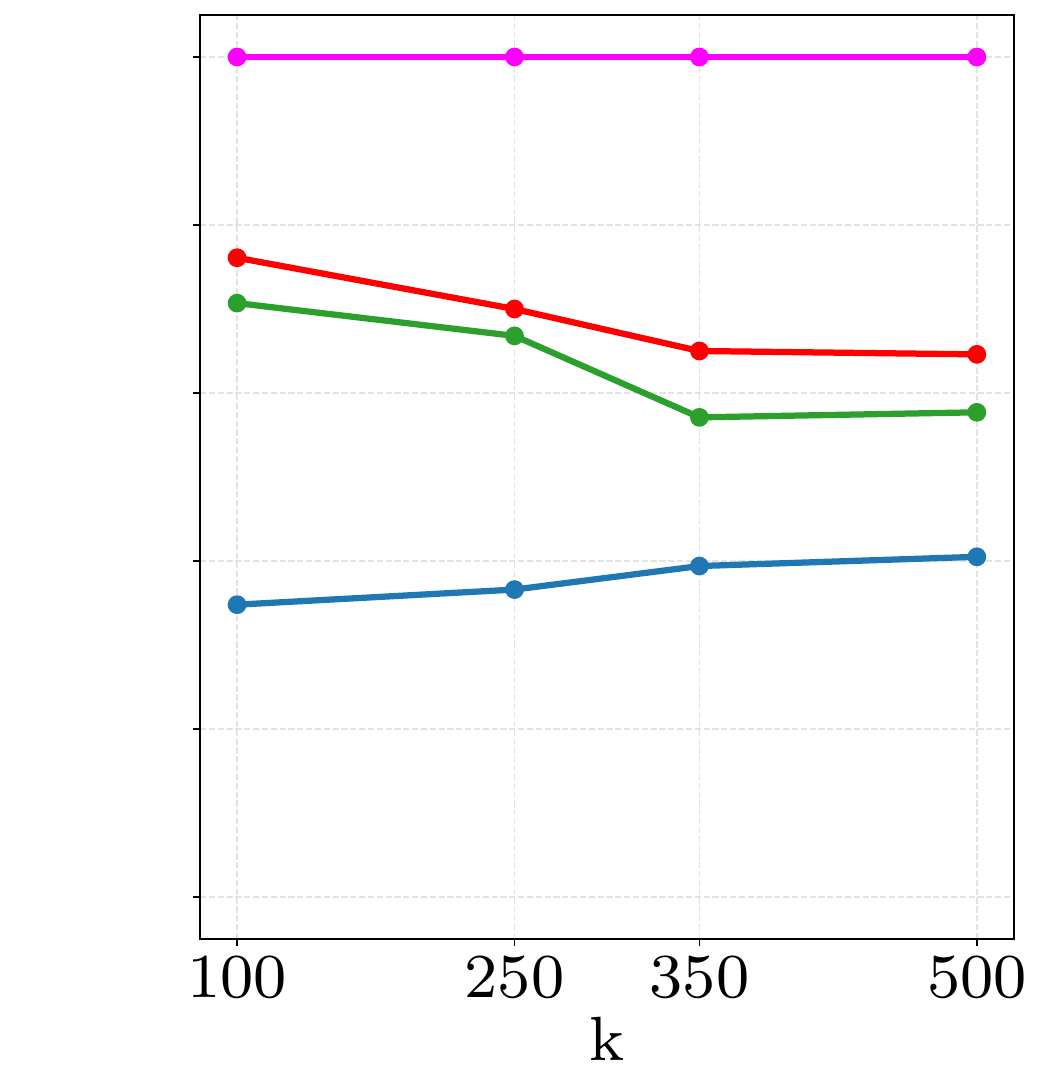}
        \caption{$n=750$}
    \end{subfigure}
    \hfill
    \begin{subfigure}[b]{0.24\linewidth}
        \centering
        \includegraphics[width=\linewidth]{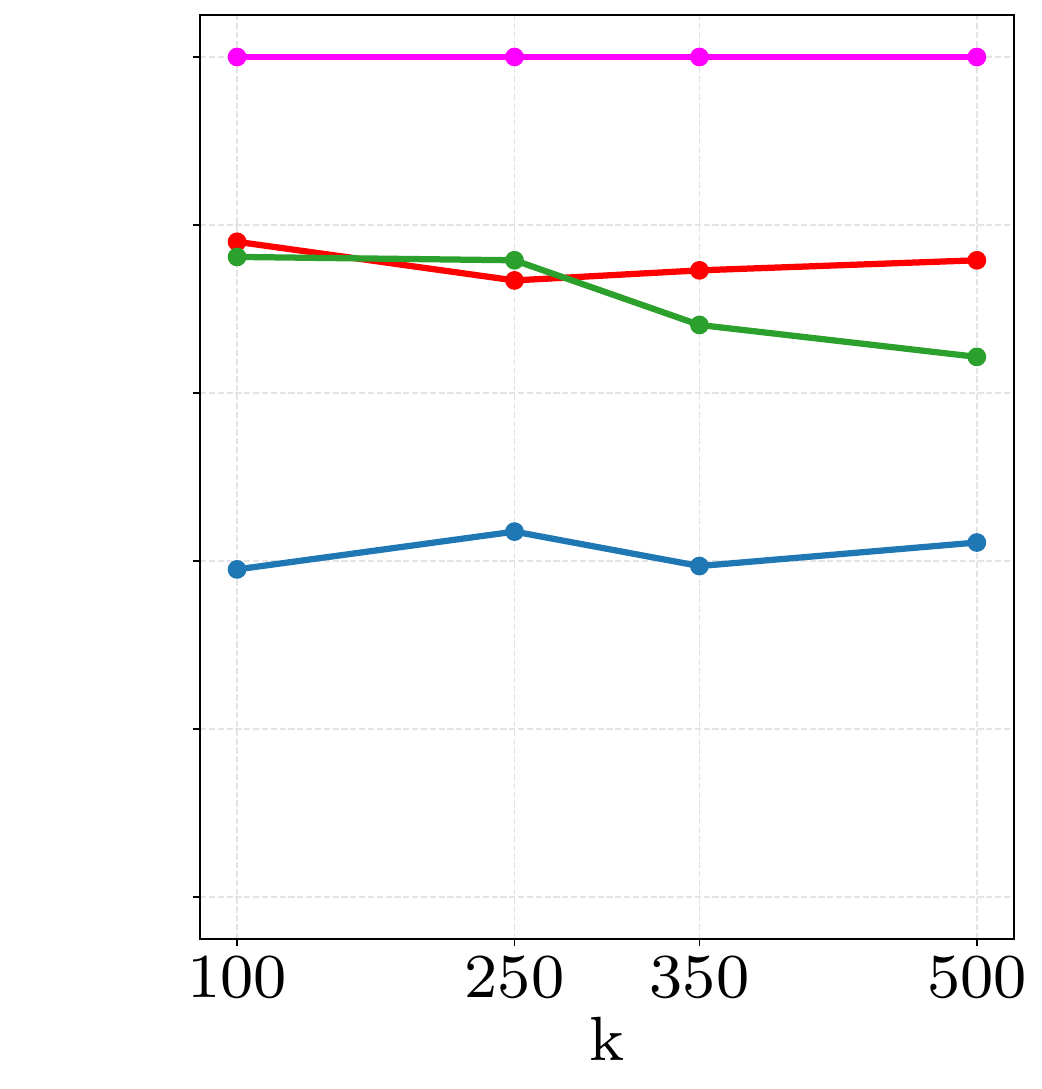}
        \caption{$n=1000$}
    \end{subfigure}

    \vspace{0.3cm} 

    \centering
    \includegraphics[width=0.8\linewidth]{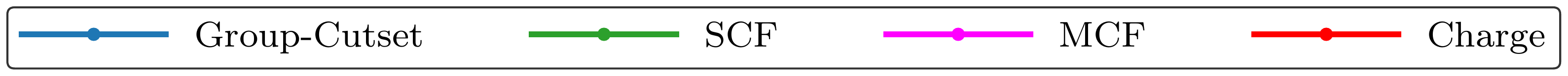}

    \caption{Optimality gap on small-scale simulated instances after $500\,\mathrm{s}$ for varying graph sizes ($n$) and number of POIs ($k$).}
    \label{fig:small_scale_comp}
\end{figure}

For instances in which the \algname{MCF} solver fits within available memory, it consistently solves the problem to optimality. This behavior is expected, as the \algname{MCF} solver provides a tight and complete encoding of connectivity and group-coverage constraints, making it the preferred choice whenever problem size permits.

In contrast, the more compact solvers (\algname{Charge} and \algname{SCF}) exhibit significantly larger optimality gaps on these instances. The Group-cutset formulation, however, maintains high performance even at small scales, indicating a degree of scale insensitivity.

One possible explanation for the weaker performance of the Charge and SCF formulations on these simulated instances, despite their competitive results on some larger real-world graphs (e.g., \texttt{Bridge-1000}), is the relatively small number of POIs. With fewer groups, coverage constraints become less restrictive, enlarging the feasible region of the LP relaxation  and making tight lower bounds harder to obtain. 
In such settings, formulation strength, measured by lower-bound tightness, becomes the dominant factor, where Group-cutset excels. This interpretation is supported by the observed trend that the optimality gaps for the \algname{Charge} and \algname{SCF} solvers decrease as the number of POIs increases, whereas the gap slightly increases for \algname{group-cutset} solvers.

\subsection{Evaluation of Primal Heuristic}
\label{heuristic-eval}
We evaluate the effectiveness of our problem-specific \gip heuristics discussed in App.~\ref{appendix:gip-heuristic}, and compare them against the heuristic proposed by \citet{mizutani2024leveraging}, referred to as the \emph{ST heuristic}. The ST heuristic first selects vertices greedily based on proximity to the root by finding the appropriate shortest-paths, until a group-covering set is obtained, then constructs a Steiner tree using a minimum spanning tree, which yields a 2-approximation of the optimal Steiner tree connecting those chosen vertices. Finally, their approach produces a tour by traversing the tree back and forth. We note that as presented, this heuristic is only applicable in an undirected graph regime. 

\begin{figure}[t]
    \centering
    \begin{subfigure}[t]{0.48\linewidth}
        \centering
        \includegraphics[width=\linewidth]{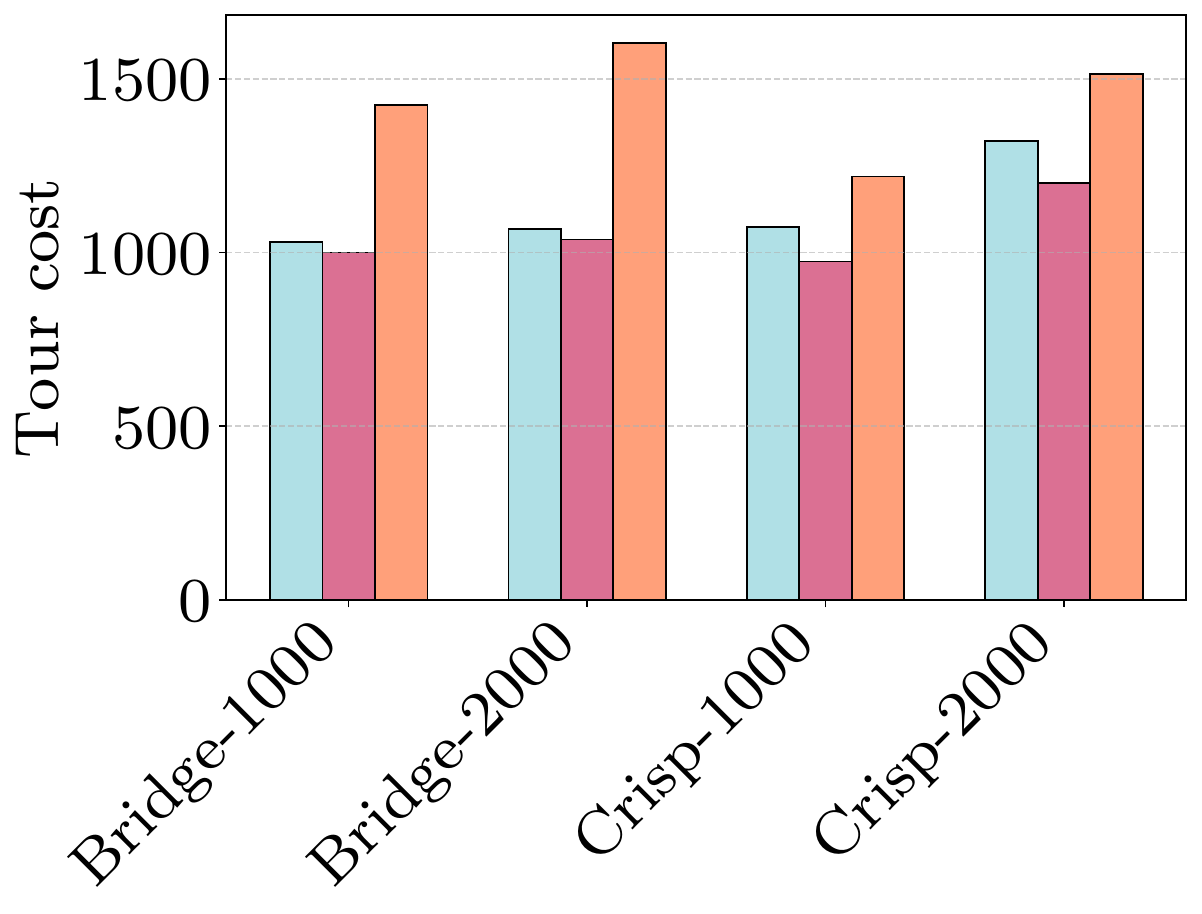}
        \caption{} 
        \label{fig:heuristics-tourcost}
    \end{subfigure}
    \hfill
    \begin{subfigure}[t]{0.48\linewidth}
        \centering
        \includegraphics[width=\linewidth]{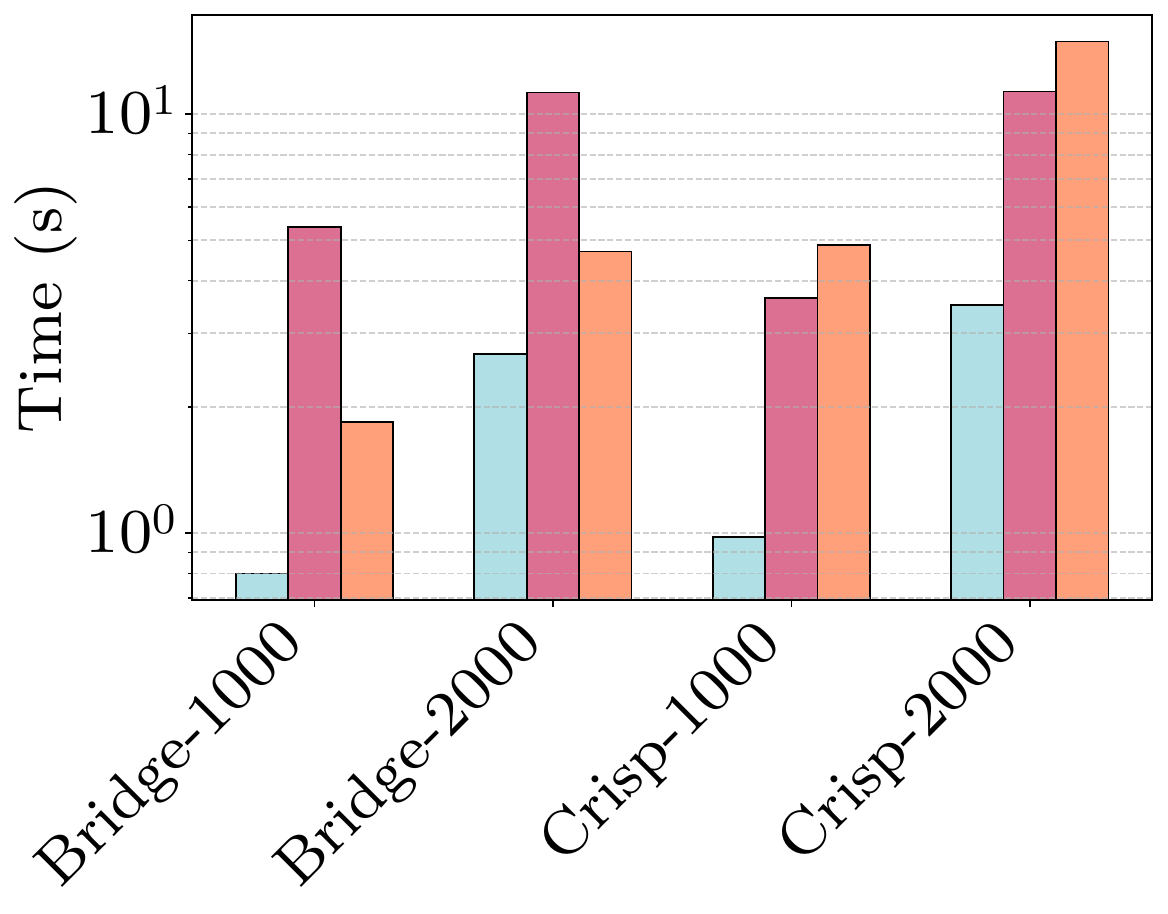}
        \caption{} 
        \label{fig:heuristics-runtime}
    \end{subfigure}
    \hfill

    \vspace{2mm} 
    
    \includegraphics[width=1\linewidth]{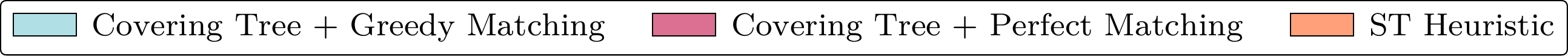}

    \caption{
    Heuristic evaluation.
    \textbf{(a)} Tour cost comparison at the root of the BnC tree.
    Tour costs on \texttt{CRISP} instances are scaled by $1,000 \times$.
    \textbf{(b)} Execution runtime in seconds.}
    \label{fig:ph-ablation1}
\end{figure}

\begin{figure}[h]
    \centering
    \includegraphics[width=0.6\linewidth]{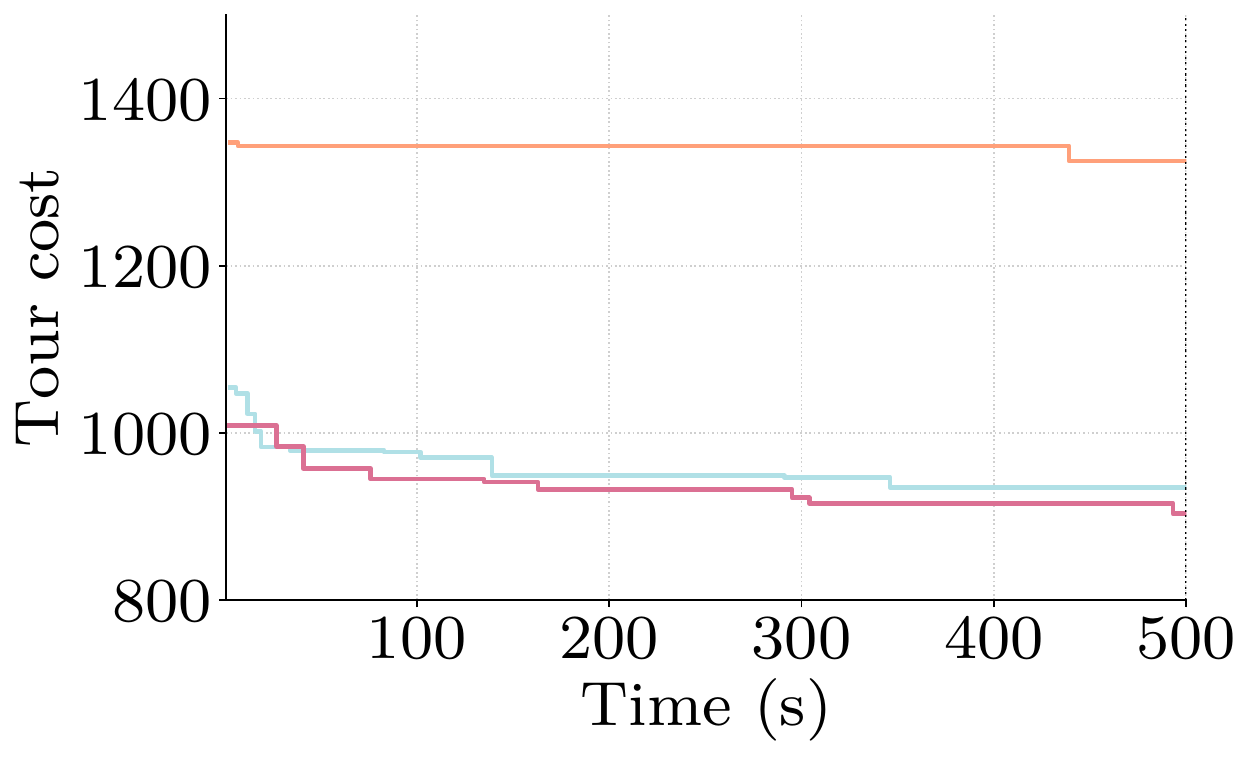} \\
    
    \vspace{2mm} 
    
    \includegraphics[width=1\linewidth]{Illustrations/Legend_Horizontal.pdf}

    \caption{
        Anytime performance on the \texttt{Bridge-1000} instance. 
        The plot shows the tour cost improvement over time for the different heuristics.
    }
    \label{fig:ph-ablation2}
\end{figure}

In contrast, our \gip heuristic builds the group-covering tree iteratively, minimizing the incremental edge cost while covering new groups.
This integrates vertex selection and tree construction into a single phase,
which tends to yield more compact covering trees.
Furthermore, the tree-to-tour conversion goes beyond simple back-and-forth
traversal by augmenting the tree to an Eulerian graph and extracting an Eulerian tour.

We evaluate two variants of our ``Covering-Tree'' heuristic on real-world scenarios, which differ
only in the final tour-construction phase (App.~\ref{appendix:gip-heuristic}):
one using minimum-weight \emph{perfect matching} and the other using \emph{greedy matching}.
Fig.~\ref{fig:ph-ablation1} compares the resulting heuristic solutions\footnote{To assess the intrinsic performance of the heuristics, both
Covering-Tree variants are evaluated without LP-based cost discounting, i.e., using the original edge costs of the instance graph.} on the
\texttt{CRISP} and \texttt{Bridge} instances.
Specifically, Fig.~\ref{fig:ph-ablation1}(a) reports the cost of the heuristic-generated tour. In Fig.~\ref{fig:ph-ablation1}(b) we report the running time of a single heuristic execution on the scenario problem instance, i.e. independently from the BnC solver.

Finally, Fig.~\ref{fig:ph-ablation2} illustrates a representative
\algname{Group-Cutset} solver run for the \texttt{Bridge-1000} instance,
augmented with each heuristic variant, over a time horizon of $t=500$ seconds.

Comparing the two Covering-Tree variants reveals a modest trade-off between solution quality and runtime (Fig.~\ref{fig:ph-ablation1}).
Greedy matching substantially reduces heuristic computation time, at the cost of a mild degradation in solution quality of up to approximately $5\%$.

For solvers with sufficiently long time limits, this degradation has negligible impact on the final solution quality. Fig.~\ref{fig:ph-ablation2} illustrates this behavior, showing that both variants, when used as primal heuristics within the BnC framework, consistently outperform the ST heuristic, with no clear dominance between them in this regime.

However, for the larger instances considered in Sec.~\ref{sec:large-eval},
heuristic runtime becomes a dominant factor.
In this regime, the faster greedy-matching variant provides a clear advantage, a trend we consistently observed during experimentation but do not report explicitly for brevity.
Accordingly, we adopt the greedy-matching variant within the \algname{Group-Cutset} BnC solver for all experimental evaluations
(Sec.~\ref{sec:evaluation}).

Comparing these results to those obtained using Gurobi’s generic internal heuristics when applied to the \algname{Group-Cutset} formulation highlights
the importance of problem-specific primal heuristics.
In additional experiments conducted during evaluation (not shown for brevity),
our heuristic consistently produced substantially better incumbent solutions than Gurobi’s default internal heuristics.
On the \texttt{CRISP-1000} instance, where the difference is smallest, incumbent costs were approximately five times lower, while on more challenging instances
the improvement reached up to two orders of magnitude.
These results demonstrate the critical role of tailored primal heuristics for obtaining high-quality incumbents in lazy-constraint settings.

\end{document}